\documentclass[]{elsarticle}

\journal{Artificial Intelligence}

\usepackage{paralist,xcolor,verbatim,graphicx,amsmath,amssymb,stmaryrd,textcomp,proof}

\usepackage[shortlabels]{enumitem} 

\usepackage{lineno,hyperref} 
\modulolinenumbers[5]

\usepackage{amsthm}
\newtheorem{definition}{Definition} 
\newtheorem{theorem}{Theorem}

\usepackage{listings,soul}

\newcommand{\logikey}{\textsc{LogiKEy}}

\newcommand{\leoII}{\textsc{LEO-II}}
\newcommand{\leoIII}{\textsc{Leo-III}}
\newcommand{\isabellehol}{\textsc{Isabelle/HOL}}

\newcommand{\typearrow}{\shortrightarrow} \newcommand{\itype}{i}
\newcommand{\utype}{u} \newcommand{\proptype}{\tau}

\biboptions{sort&compress}

\begin{document}

\begin{frontmatter}

  \title{Designing Normative Theories for Ethical and Legal Reasoning:
    \logikey\ Framework, Methodology, and Tool Support}
  \fntext[myfootnote]{Benzm\"uller was funded by the
    VolkswagenStiftung under grant CRAP (Consistent Rational
    Argumentation in Politics). Parent and van der Torre were
    supported by the European Union's Horizon 2020 research and
    innovation programme under the Marie Sk\l{}odowska-Curie grant
    agreement MIREL (MIning and REasoning with Legal texts) No
    690974.}

  \author[addressBer,addressLux]{Christoph Benzm\"uller}
  \ead{c.benzmueller@fu-berlin.de} \author[addressLux]{Xavier Parent}
  \ead{x.parent.xavier@gmail.com} \author[addressLux,addressChina]{Leendert
    van der Torre} \ead{leon.vandertorre@uni.lu}

  \address[addressLux]{Department of Computer Science, University
    of Luxembourg, Esch-sur-Alzette, Luxembourg}
  \address[addressBer]{Department of Mathematics and Computer Science,
    Freie Universit\"at Berlin, Berlin, Germany}
  \address[addressChina]{Institute of Logic and Cognition, Zhejiang
    University, Hangzhou, China}

\begin{abstract}
  A framework and methodology---termed \logikey---for the design and
  engineering of ethical reasoners, normative theories and deontic
  logics is presented.  The overall motivation is the development of
  suitable means for the control and governance of intelligent
  autonomous systems.  \logikey's unifying formal framework is based
  on semantical embeddings of deontic logics, logic combinations and
  ethico-legal domain theories in expressive classic higher-order
  logic (HOL). This meta-logical approach enables the provision of
  powerful tool support in \logikey: off-the-shelf theorem provers and
  model finders for HOL are assisting the \logikey\
  designer of ethical intelligent agents to flexibly experiment with
  underlying logics and their combinations, with ethico-legal domain
  theories, and with concrete examples---all at the same
  time. Continuous improvements of these off-the-shelf provers,
  without further ado, leverage the reasoning performance in \logikey.
  Case studies, in which the \logikey\ framework and methodology has
  been applied and tested, give evidence that HOL's undecidability
  often does not hinder efficient experimentation.
\end{abstract}

\begin{keyword} Trustworthy and responsible AI; Knowledge
  representation and reasoning; Automated theorem proving; Model
  finding; Normative reasoning; Normative systems; Ethical issues; Semantical embedding; Higher-order logic
\end{keyword}

\end{frontmatter}


\section{Introduction}

The \emph{explicit representation-of and reasoning-with ethical and legal knowledge} capacitates ethical intelligent systems with
increasing levels of autonomy in their decision making
\cite{Malle2016,MalleEtAl2017}. It also supports sufficient degrees of
reliability and accountability, and it enables human intuitive
interaction means regarding explainability and transparency of
decisions.  The ethical and/or legal theories---we often call them
{\em ethico-legal theories} in the remainder---can thereby be
represented as normative systems formalized on top of suitably
selected underlying deontic logics
\cite{deon:handbook,nmas,textbook18} and logic combinations
\cite{LogicCombining,gabbay2003many}.  In this article we introduce
our \emph{\textbf{Logi}c and \textbf{K}nowledge \textbf{E}ngineering
  Framework and Methodolog\textbf{y}}, termed \logikey, to apply
off-the-shelf theorem proving and model finding technology to the
design and engineering of specifically tailored normative systems and
deontic logics for deployment in ethical intelligent reasoners.

An overall objective of \logikey\ is to enable and support {\em the
  practical development of computational tools for normative reasoning
  based on formal methods.}  To this end we introduce the concept of
experimentation in the formal study of ethical and legal reasoning: by
representing examples, ethico-legal domain theories, deontic logics
and logic combinations in a computational system we enable predictions
and their assessment, and apply formal methods. Given this motivation
to experiment-with, explore and assess different theories and logics
in the design of ethical intelligent reasoners we address the
following research questions:
\begin{enumerate}
\item[A:]  Which \emph{formal framework to choose},
\item[B:]  which \emph{methodology to apply}, and
\item[C:]  which \emph{tool support to consult}
\end{enumerate}
for experimentation with ethico-legal theories?

Several examples of computational tools to experiment with ethical
reasoners, normative systems and deontic logics have been introduced
in the recent literature~\cite{J45,C71,J46,MederMasters,C76,C77} and
are used in this article to illustrate the feasibility of the proposed
solutions.  \logikey\ extracts the general insights from these prior
case studies and bundles them into a coherent approach.

\subsection{Formal framework: expressive classical higher-order logic}
We first explain why we use classical higher-order logic (HOL), i.e.,
Church's type theory \cite{J43}, as our {formal framework}.  To
understand the challenge of the choice of a formal framework (cf.~our
research question~A), consider the requirements of a typical
ethical agent architecture as visualized in Fig.~\ref{fig:topdown}.
\begin{figure}[t] \centering
  \includegraphics[width=.62\textwidth]{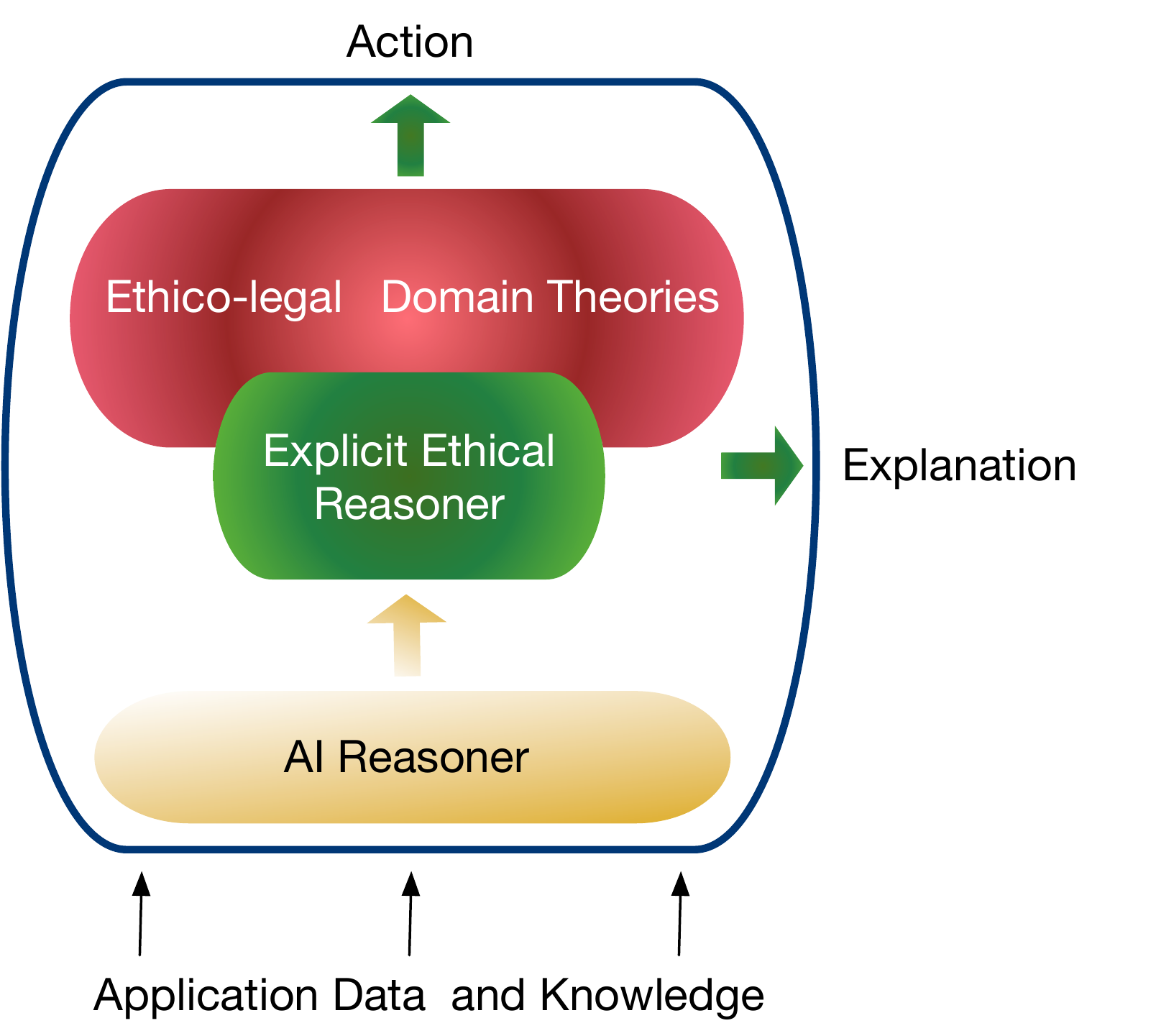}
  \caption{Explicit ethical reasoner for intelligent autonomous
    systems\label{fig:topdown}.}
\end{figure}
The displayed architecture for an intelligent autonomous system
\cite{IAS} with explicit ethical competency distinguishes an explicit
ethical reasoner and ethico-legal domain theories from an AI
reasoner/planner and from other components, including also application
data and knowledge available to both reasoners.  The ethical reasoner
takes as input suggested actions from the AI reasoner/planner, hints
to relevant application data and knowledge, and the ethico-legal
domain theories, and it produces as output assessments and judgements
concerning which actions are acceptable or not, and it also provides
corresponding explanations. That is, the actions suggested by the AI
reasoners in Fig.~\ref{fig:topdown} are not executed immediately, but
additionally assessed by the ethical reasoner for compliance with
respect to the given ethico-legal domain theories. This assessment is
intended to provide an additional, explicit layer of explanation and  control on top of
the AI reasoner, which ideally already comes with solid own ethical
competency.  For the aims in this article, the details of the ethical
agent architecture are not important. For example, it does not matter
whether computations on the level of the AI reasoner are based on
sub-symbolic or symbolic techniques, or combinations of them, since
the suggested actions are not executed immediately---at least not
those considered most critical. Instead they are internally assessed,
before execution, against some explicitly modelled ethico-legal 
theories.  These theories govern and control the behaviour of the
entire system, and they also support---at this upper-level---system
verification and intuitive user explanations.  What counts, in
particular, in highly critical contexts, is not only how an AI
reasoner computed the critical action it wants to perform, but in
particular whether this action passes the additional assessment by the
upper-level explicit ethical reasoner before its execution.

The formal framework of the ethical reasoner has to be populated with
ethico-legal domain theories, be able to combine logics, reason about
the encoded domain theories, and experiment with concrete applications
and examples, as visualized in~Fig.~\ref{fig:metho}.
\begin{figure}[t]\centering
  \includegraphics[scale=0.52]{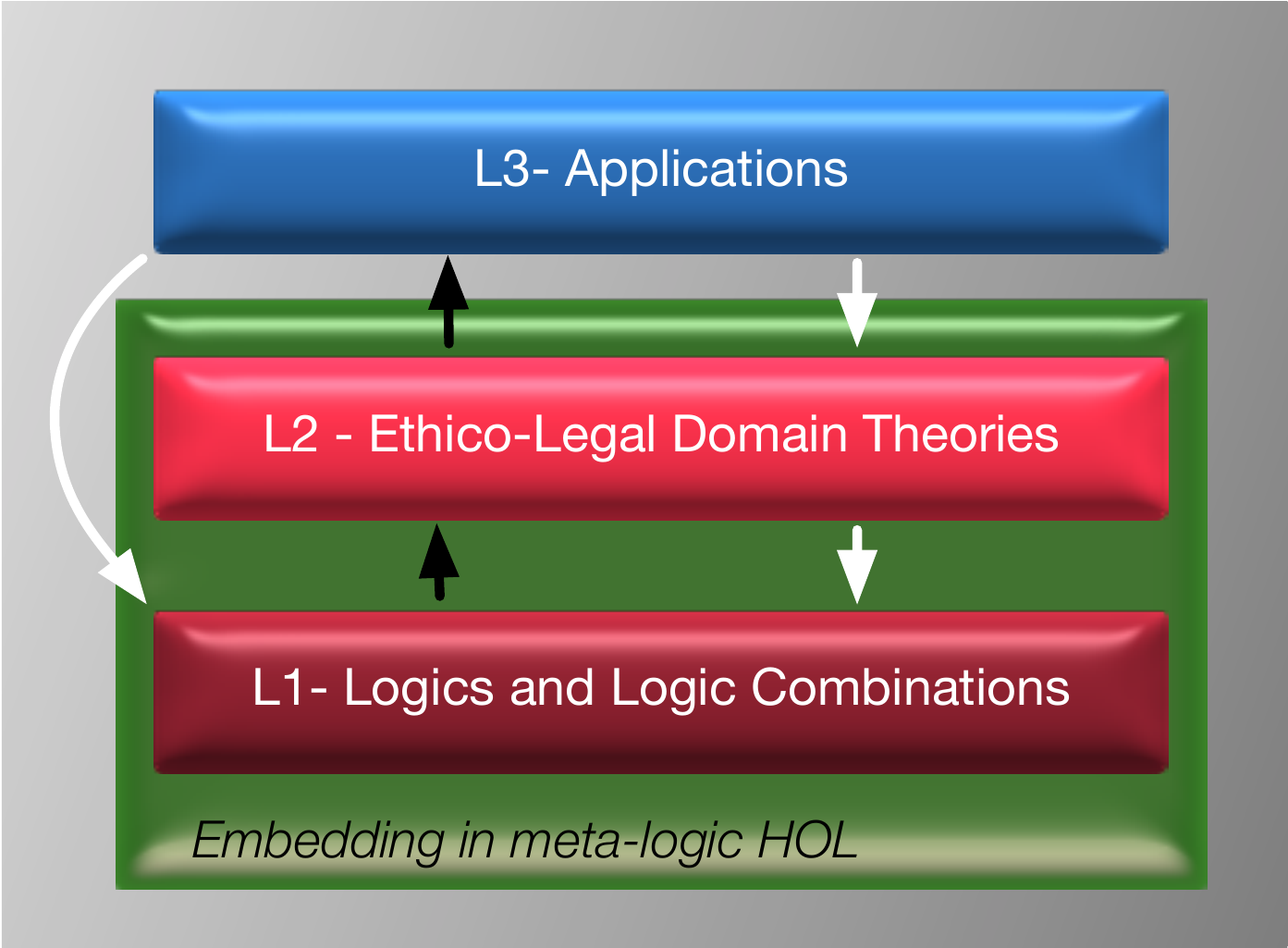}
\caption{Logic and knowledge engineering in \logikey.
  \label{fig:metho}}
\end{figure}

This leads to an apparent paradox: on the one hand, the area of
knowledge representation and reasoning in artificial intelligence is
built on highly specialized and efficient formalisms;\footnote{In
  formalized mathematics, in contrast, practical
  developments are interestingly progressing already in the opposite
  direction; proof assistants such as \textsc{Lean} \cite{Lean},
  \textsc{Coq} \cite{Coq}, or \textsc{Agda} \cite{Agda}, for example,
  are being developed and deployed, which, due to their rich type
  systems, are practically even more expressive than the HOL provers
  we utilize in \logikey: \isabellehol~\cite{Isabelle} and
  \leoIII~\cite{Leo-III}. The latter provers, however, provide a very
  good degree of proof automation, and this, combined with a
  sufficient degree of expressivity for our purposes, is the reason we
  prefer them so far.}  on the other hand to combine all these aspects
in a single formalism and to experiment with it, we need a highly
expressive language.  For example, as we explain in more detail in
Sect.~\ref{sec:deontic-logic}, the handling of normative concepts such
as obligation, permission and prohibition
is actually far more complex than one would initially think. This is
illustrated, among others, by the notorious paradoxes of normative
reasoning~\cite{Carmo2002}, and appropriate solutions to this
challenge require sophisticated deontic logics.  In \logikey\ we
therefore model and implement ethical reasoners, normative systems and
deontic logics in the expressive meta-logic HOL. This often triggers
scepticism, in particular, in the AI knowledge representation and
reasoning community, since HOL has known theoretical drawbacks, such
as undecidability.  However, HOL, just like first-order logic, has
decidable fragments, and our shallow semantical embedding~(SSE)
approach, cf.~Sect.~\ref{sec:SSE}, is capable of translating decidable
proof problems, say in propositional modal logic, into corresponding
decidable proof problems in (a fragment of) HOL. For example, the
modal formula $\Box
(\varphi \wedge
\psi)$ is---assuming a request about its global validity---rewritten
in the SSE approach into the guarded fragment formula $\forall
w. \forall v. \neg (R w v) \vee (\varphi v \wedge \psi v)$, where
$R$
is denoting the accessibility relation associated with $\Box$.
The (in-)validity of formulas in the guarded fragment is decidable, and our example formula can be effectively handled by
\isabellehol and \leoIII. The reason
why these system can handle such decidable problems effectively is
simple: they internally collaborate with state-of-the-art SAT \& SMT
(satisfiability modulo theories) solvers, and with first-order
automated theorem proving systems; this ensures that fragments
of HOL are indeed attacked with tools that are specialized for these
fragments. Moreover, the framework we provide is universal, and this
makes it an ideal approach for experimentation.

\subsection{Methodology: \logikey}

The \logikey\ methodology, a high-level development protocol that will
be explained in detail in~Sect.~\ref{sec:logikey}, constitutes our
answer to research question B. This methodology supports and
guides the simultaneous development of the three layers depicted in
Fig.~\ref{fig:metho}: combining logics, ethico-legal theories, and
concrete examples and applications.  The black arrows between these
three levels symbolize dependencies. The normative governance
applications developed  at layer L3 depend on ethico-legal domain
theories imported from layer L2, which in turn are formalized within a
specific logic or logic combination provided at layer L1.\footnote{In
  this article we utilize a uniform coloring scheme in all our
  conceptual figures: layers L1 and L2 are displayed using dark red
  and light red coloring, respectively. Applications and experiments
  are displayed in blue, and the meta-logical HOL framework that is
  supporting the layers L1 and L2 is shown in green.}  The
engineering process at the different layers has backtracking points
and several work cycles may be required; thereby the higher layers may
also pose modification requests to the lower layers.  The white arrows
symbolize such possible requests, and they may, unlike in most other
approaches, also include far-reaching modifications of the logical
foundations engineered at layer L1; such modifications at the logic
layer are flexibly facilitated in our meta-logical approach.

As a concrete application, consider a smart home visualized in
Fig.~\ref{ama} \cite{DBLP:journals/corr/abs-1812-04741}.
\begin{figure}[htp]\centering
  \includegraphics[scale=0.5]{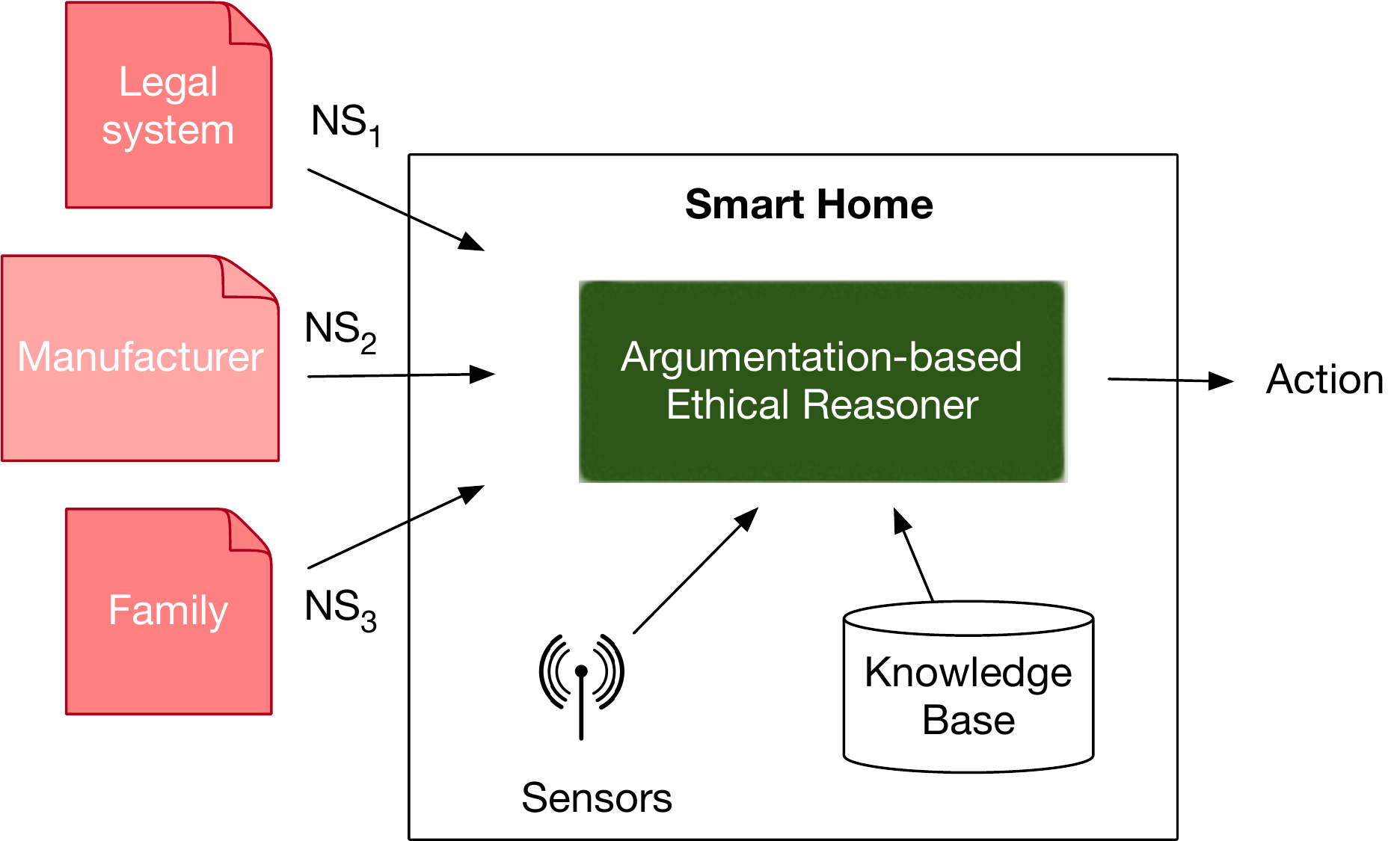}
  \caption{Example of ethical reasoner in a smart home
    \cite{DBLP:journals/corr/abs-1812-04741}.\label{ama}}
\end{figure}
This one uses an argumentation engine \cite{Dung} in order to make
ethical decisions based on normative systems of stakeholders.  The
argumentation-based mechanism is used to find an agreement in case of
moral dilemmas, and to offer an explanation as to how a specific
morally sensitive decision has been made.  Liao {\em et al.}~assume
the existence of pre-defined systems of norms, coming from the outside
and constraining the agent's behavior. Like humans, intelligent
systems evolve in a highly regulated environment.  For instance, a
smart home processes personal data 24/7, and, as will be discussed in
Sect.~\ref{sec:GDPR}, there are legal rules one must comply with when
processing personal data; in the European Union such rules include the
General Data Protection Regulation (GDPR, Regulation EU 2016/679).
Given a description of the current situation, the smart home should
thus be able, for example, to check compliance with these rules, and
act accordingly in case of non-compliance.

\subsection{Tool support: \isabellehol\ and \leoIII}

Our use of higher-order tool support, cf.~our research question C,
is visualized in Fig.~\ref{fig:infrastructure}. The \logikey\
methodology supports experimentation with different normative
theories, in different application scenarios, and it is not tied to a
specific deontic logic and also not restricted to decidable logics
only.  Since ethico-legal theories as well as suitable normative
reasoning formalisms are both subject to exploration in \logikey, we
introduce a flexible workbench to support empirical studies in which
both can be varied, complemented, assessed and compared.
\begin{figure}[t]\centering
  \includegraphics[width=.5\textwidth]{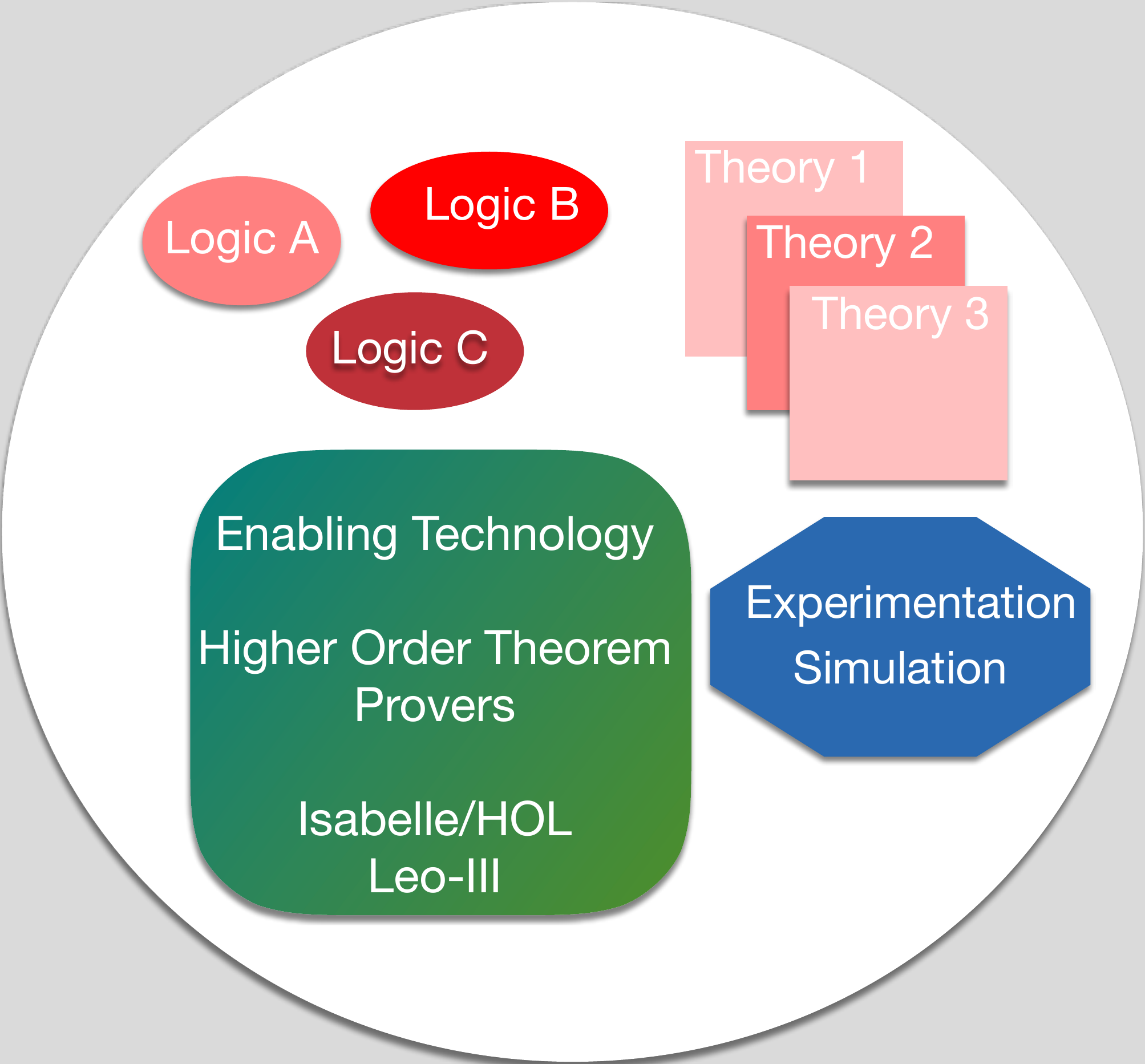}
	\caption{Flexible tool support for normative reasoning is required in \logikey. \label{fig:infrastructure}}
      \end{figure}

      The enabling technology are higher-order theorem proving systems
      via the SSE technique~\cite{J41}. We benefit from good
      improvements in the past decade in interactive theorem proving
      (ITP) and automated theorem proving (ATP) for HOL, and also from
      the overall coalescence of heterogeneous theorem proving
      systems, as witnessed, in particular, by \isabellehol,
      \leoII~\cite{J30} and \leoIII, which fruitfully integrate other
      specialist ATP systems.  In this way---as a sort of relevant byproduct
      of our research---we build a bridge between the classical and
      non-classical logics communities, the deduction systems
      community, and the formal ethics community;\footnote{The formal
        ethics community, among others, includes: \emph{Normative
          Multi-Agent Systems} (\url{http://icr.uni.lu/normas}),
        \emph{Deontic Logic and Normative Systems}
        (\url{http://deonticlogic.org}) and \emph{Formal Ethics}
        (\url{http://www.fe2019.ugent.be}).  With deduction systems
        community we refer to the areas of \emph{Automated Reasoning}
        (\url{http://aarinc.org}) and \emph{Interactive Theorem
          Proving} (\url{https://itp2018.inria.fr}); these are in turn
        composed of several subcommunities.}  cf.~also
      Fig.~\ref{fig:framework}, which will be explained in more detail
      in Sect.~\ref{sec:SSEapproach}.

      Theorem proving was a major impetus for the development of
      computer science, and ATP and ITP systems have been applied to a
      variety of domains, including mathematics and software \&
      hardware verification. However, comparably little work has been
      invested so far to apply such systems to deontic logic and
      normative reasoning.

\subsection{Contributions}
The core contributions of this article are on various levels:
\begin{enumerate}
\item Motivated by our position in the trustworthy AI debate---where
  we defend the need for explicit ethico-legal governance of
  intelligent autonomous systems---we provide a survey of various
  technical results pertaining to the application of automated theorem
  proving to ethical reasoners, normative systems and deontic logics.
   
\item This survey is targeted to a general audience in AI.  Instead of
  presenting selected formal definitions and proofs, we discuss
  motivations, commonalities and applications. Moreover, we provide an
  entry-level user-guide to our framework and methodology, called
  \logikey.  Existing approaches~\cite{DBLP:journals/ijswis/KontopoulosBGA11,DBLP:conf/birthday/FurbachS14,DBLP:conf/isaim/BringsjordGMS18,DBLP:series/sapere/PereiraS16}
  are always tied to rather specific logical settings. We introduce a
  flexible workbench to support empirical studies with such theories
  in which the preferred logic formalisms themselves can be varied,
  complemented, assessed and compared.
    
\item Our ambition is to build a library of ethical reasoners,
  normative systems and deontic logics, and with this article we make
  a first step in that direction by providing a collection of \isabellehol\
  source files that can be reused by researchers and students
  interested in adopting \logikey\ in their own work.\footnote{Our data set is available at \url{http://logikey.org}; see also the associated article \cite{J53} in \textit{Data in brief}.}
\end{enumerate}

We briefly clarify some terminology as used in the article. With
\emph{experimentation} we refer to the action or process of trying out
new ideas, methods, or activities; this can be, for instance, the
drafting of a new logic or combination of logics, or a new ethical
theory, where one may want to see what would be the consequences of
its adoption in a given application context. \emph{Methodology} refers
to the principles underlying the organization and the conduct of a
design and knowledge engineering process; hereby \emph{design} means
the depiction of the main features of the system we want to achieve,
and \emph{(knowledge or logic) engineering} refers to all the
technical and scientific aspects involved in building, maintaining and
using a knowledge-based, resp.~logic-based, system.
\emph{Infrastructure} is the framework that supports different
components and functionalities of a system, and by \emph{logic
  implementation} we here mean an engineering process that returns an
executable computer program, a theorem prover, for a given logic; this
can be done in very different ways, for example, by coding software
from scratch or, as we prefer, by adopting the SSE approach, which
utilizes HOL as a meta-level programming language within existing
reasoners such as \isabellehol\ or \leoIII.

The structure of this article is as follows.  Section~\ref{sec:SSE}
further motivates the need for a flexible normative reasoning
infrastructure and presents \logikey's expressive reasoning
framework. Section~\ref{sec:deontic-logic} briefly surveys and
discusses the challenging area of normative reasoning; this section
also explains which particular deontic logics have already been
implemented. A concrete example implementation of a state-of-the-art
deontic logic, \AA qvist system {\bf E}, is presented in further
detail in Sect.~\ref{sec:example-section}.  Tool support is discussed
in Sect.~\ref{sec:tools}, and subsequently two case studies are
presented in Sect.~\ref{sec:case-studies}.  The first case study
illustrates contrary-to-duty compliant reasoning in the context of the
general data protection regulation.  A second, larger case study shows
how our framework scales for the formalization and automation of
challenging ethical theories (Gewirth's \emph{``Principle of Generic
  Consistency"} \cite{GewirthRM}) on the computer. To enable the
latter work an extended contrary-to-duty compliant higher-order
deontic logic has been provided and utilized in our framework.
Section~\ref{sec:logikey} provides a detailed description of
\logikey's 3-layered logic and knowledge engineering methodology; it
may thus provide useful guidance to future \logikey\ users.  Sections
\ref{sec:RelatedWork} and \ref{sec:FurtherWork} discuss related work
and further research, and Sect.~\ref{sec:Conclusion} concludes the
article.

\section{The \logikey\ expressive reasoning framework: meta-logic
  HOL} \label{sec:SSE}

This section presents and discusses \logikey's expressive reasoning
framework, cf.~Fig.~\ref{fig:framework}, which utilizes and adapts
Benzm\"uller's SSE approach \cite{J41}.
 
\begin{figure}[t]
  \hskip1em
  \includegraphics[width=.95\textwidth]{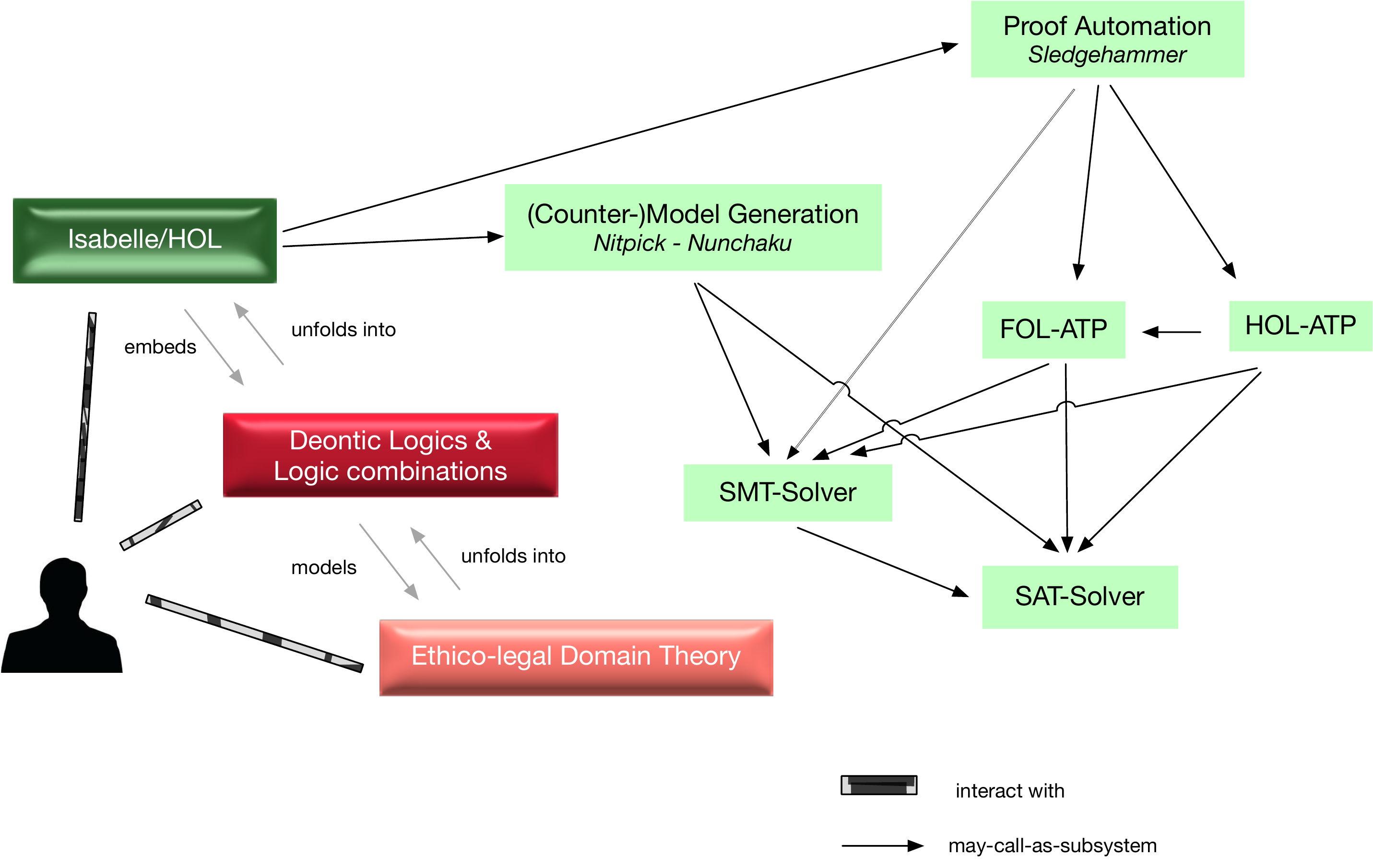}
\caption{The \logikey\ expressive reasoning framework is based on meta-logic HOL. \label{fig:framework}}
\end{figure}

The implementation of specialist theorem provers and model finders
for, e.g., ambitious deontic logics and their extensions and
combinations, is very tedious and requires expertise. \logikey\
therefore focuses on reuse and adaptation of existing technology
rather than new implementations from first principles. A particular
interest is to enable novices to become acquainted with the area of
normative reasoning in short time and in a computer-assisted, hands-on
fashion; this should enable them to gradually acquire much needed
background expertise.  An initial goal of our work therefore has been
to build up a starting basis of mechanized deontic logics at layer L1
and ethico-legal domain theories at layer L2, and to make them
accessible for experimentation to students and researchers. Our
approach pays much attention to intuitive interaction within
sophisticated user-interfaces to access, explore, assess and modify
both the used foundational logic and logic combination and the encoded
domain theories. The tools we reuse are state-of-the-art reasoning
systems for classical logic that are actively developed by the
deduction systems community. That relieves us from resource intensive
investments in the implementation and maintenance of new
technology. However, to enable this reuse of technology, a bridge
between deontic logics (and their combinations with other modal
logics) and classical logic was needed. This bridge is provided by the
SSE approach.
 
The framework, methodology and technology we contribute requires some
modest background knowledge to be acquired.
The logic developers' perspective, as relevant at \logikey\ layer L1,
is a bit more ambitious and requires some familiarity with meta-logic
HOL (cf.~Sect.~\ref{sec:HOL}), and also a good mastery of the SSE
technique (cf.~Sect.~\ref{sec:SSEapproach}).  However, also at that
level we support the adoption of initial skills by providing a library
of example encodings of deontic logics, and other relevant logics and
logic combinations, to start with; this is helpful, since it, among
others, enables copying and pasting from these encodings. Several
successful student projects at BSc, MSc and PhD level meanwhile
provide good evidence for the practical relevance of our approach at
the developers level~\cite{J46,MederMasters,C76,C77,C62,zahoransky19,GewirthProofAFP,C82,J38,J47,J50}.

\subsection{Logic engineering and implementation methodology: SSEs in
  HOL} \label{sec:SSEapproach}

Our expressive reasoning framework, depicted graphically in
Fig.~\ref{fig:framework}, is based on the SSE
approach~\cite{J23,J41}. HOL is utilized in this approach as a
universal meta-logic in which different deontic logics are
semantically embedded in a \emph{shallow} way by directly encoding
their semantics in meta-logic HOL; commonalities in the semantics of
both logics are thereby typically shared as much as
possible.\footnote{\label{foot:deep}\emph{Shallow semantical
    embeddings} are different from \emph{deep embeddings} of a target
  logic. In the latter case the syntax of the target logic is
  represented using an inductive data structure (e.g., following the
  definition of the language).  The semantics of a formula is then
  evaluated by recursively traversing the data structure, and
  additionally a proof theory for the logic maybe be encoded. Deep
  embeddings typically require technical inductive proofs, which
  hinder proof automation, that can be avoided when shallow semantical
  embeddings are used instead. For more information on shallow and
  deep embeddings we refer to the literature
  \cite{DeepShallow,DeepShallow2}.}
  
We have extended this approach in our project for ambitious deontic
logics, their extensions and combinations with other relevant logics.
This enables the use of interactive proof assistants, such as
\isabellehol, which comes with a sophisticated user-interface and, in
addition, integrates various state-of-the-art reasoning tools. The ATP
systems integrated with \isabellehol\ via the sledgehammer
\cite{Sledgehammer} tool comprise higher-order ATP systems,
first-order ATP systems and satisfiability modulo theories (SMT)
solvers, and many of these systems internally again employ efficient
SAT solving technology.  \isabellehol\ also provides two model
finders, Nitpick \cite{Nitpick} and Nunchaku \cite{Nunchaku}.

The SSE approach utilizes the syntactic capabilities of the
higher-order theorem prover to represent the semantics of a target
logic and  to define the original syntax of the target theory
within the prover. The overall infrastructure, in combination with the
SSE approach, meets our demanding requirements regarding flexibility
along different axes; cf.~Sect.~\ref{sec:tools}.

An initial focus in the SSE approach has been on quantified modal
logics~\cite{J23}. One core aspect is that the standard translation
\cite{DBLP:books/el/RV01/OhlbachNRG01} from propositional modal logic
to first-order logic can be semantically embedded, in HOL without
requiring an external translation mechanism. The modal operator
$\Diamond$, for example, can be explicitly defined by the
$\lambda$-term
$\lambda \varphi. \lambda w. \exists v. (R w v \wedge \varphi v)$,
where $R$ denotes the accessibility relation associated with
$\Diamond$. This definition, however, can be hidden from the user, who
can construct now modal logic formulas involving $\Diamond\varphi$ and
use them to represent and prove theorems.

Most importantly, however, such an embedding of modal logic operators
in HOL can be extended to also include quantifiers.  We briefly
illustrate this idea using an example (omitting types, as above, for
better readability).  First, it is relevant to note that
$\forall x. \phi x$ is shorthand in HOL for $\Pi (\lambda x. \phi x)$,
where the logical constant symbol $\Pi$ is given an obvious semantics,
namely to check whether the set of objects denoted by
$(\lambda x. \phi x)$ is the set of all objects (of the respective
type). $\exists x. \phi x$ is hence shorthand for
$\neg \Pi (\lambda x. \neg \phi x)$.  The important and interesting
aspect thus is that additional binding mechanisms for universal and
existential quantifiers can be avoided in HOL by reusing
$\lambda$-notation. This principle can now be applied also to obtain
SSEs for quantified modal logics and, analogously, for many other
quantified non-classical logics.  For example,
$\Diamond \forall x. Px$ is represented as
$\Diamond \Pi' (\lambda x. \lambda w. P x w)$, where $\Pi'$ stands for
the $\lambda$-term
$\lambda \Phi . \lambda w . \Pi(\lambda x . \Phi x w)$ and where the
$\Diamond$ gets resolved as described above. The following series of
conversions explains this encoding in more detail:
$$\begin{tabular}{lll} \label{conversion}
    $\Diamond \forall x. P x$  & $\equiv$  & $\Diamond  \Pi' (\lambda x. \lambda w. P x w)$\\
                               & $\equiv$  & $\Diamond  ((\lambda \Phi . \lambda w . \Pi(\lambda x . \Phi x w)) (\lambda x. \lambda w. P x w))$\\
                               & $\equiv$  & $\Diamond  (\lambda w . \Pi(\lambda x . (\lambda x. \lambda w. P x w) x w))$\\
                               & $\equiv$  & $\Diamond  (\lambda w . \Pi(\lambda x . P x w))$\\
                               & $\equiv$  & $(\lambda \varphi. \lambda w. \exists v. (R w v \wedge \varphi v)) (\lambda w . \Pi(\lambda x . P x w))$\\
                               & $\equiv$  & $(\lambda
                                             \varphi. \lambda
                                             w. \neg \Pi
                                             (\lambda v . \neg
                                             (R w v \wedge \varphi v))) (\lambda w . \Pi(\lambda x . P x w))$\\
                               & $\equiv$  & $(\lambda w. \neg
                                             \Pi (\lambda v
                                             . \neg (R w v \wedge  (\lambda w . \Pi(\lambda x . P x w)) v)))$\\
                               & $\equiv$  & $(\lambda w. \neg
                                             \Pi  (\lambda v
                                             . \neg (R w v \wedge \Pi(\lambda x . P x v)))) $\\
                               & $\equiv$  & $(\lambda w. \exists v . R w v \wedge \forall x . P x v) $                             
  \end{tabular}$$
  This illustrates the embedding of  $\Diamond \forall x P x$ in HOL.\footnote{In the implementation of our approach in \isabellehol\ such conversions are hidden by default, so that the user may interact with the system at the level of the target logic and enter formulas such as $\Diamond \forall x. P x$. Definition unfolding is handled in the background, but can made visible upon request by the user.}
  Moreover, this embedding can be accompanied with different notions of
  validity. For example, we say $\Diamond \forall x. P x$ is 
  globally valid (valid for all worlds $w$) if and only if $\forall w. ((\Diamond \forall x. P x)\ w)$ holds. Local validity for a particular actual world, denoted by a constant symbol $aw$,   then
  amounts  to checking whether $((\Diamond \forall x. P x)\ aw)$ is true in HOL.

  What has been sketched above is an SSE for a first-order quantified
  modal logic K with a possibilist notion of quantification. However,
  depending on the type we assign to variable $x$ in
  $\Diamond \forall x. P x$, the sketched solution scales for
  arbitrary higher-order types. Since provers such as \isabellehol\
  and \leoIII\ support restricted forms of polymorphism, respective
  universal and existential quantifiers for the entire type hierarchy
  can be introduced with a single definition.\footnote{See,
    e.g.,~lines 14--17 in Fig.~\ref{fig:e}, where respective
    polymorphic quantifier definitions, including binder notation, are
    provided for \AA qvist's system \textbf{E}.}  Further details on
  the SSE of quantified modal logics in HOL, including a proof of
  faithfulness, are available in the literature \cite{J23}. Standard
  deontic logic (modal logic KD) can easily be obtained from this
  work. To do so we simply postulate in meta-logic HOL that the
  accessibility relation $R$ underlying the $\Diamond$ operator is
  serial.  The corresponding \textbf{D} axiom
  $\Box \varphi \supset \neg \Box \neg \varphi$, or, equivalently,
  $\neg(\Box \varphi \wedge \Box \neg \varphi)$, then becomes
  derivable as a corollary from this postulate and the SSE (and so
  does the K-schema and the necessitation rule, already in base logic
  K).  Further emendations of the presented framework to obtain
  multi-modal logics and an actualist notion of quantification have
  been proposed by Benzm\"uller {\em et al.}; cf.~\cite{J41} and the
  references therein for further details.

  \subsection{Meta-logic HOL} \label{sec:HOL} HOL has its roots in the
  logic of Frege's Begriffsschrift~\cite{frege79:_begrif}.  However,
  the version of HOL as addressed here refers to a simply typed logic
  of functions, which has been put foward by Church \cite{J43}.  It
  provides $\lambda$-notation, as an elegant and useful means to
  denote unnamed functions, predicates and sets.  Types in HOL
  eliminate paradoxes and inconsistencies.
        
  To keep this article sufficiently self-contained we briefly
  introduce HOL; the reader may want to skip this subsection and get
  back again later. More detailed information on HOL and its
  automation can be found in the literature \cite{J43,J6,B5}.

\begin{definition}[Types]
  The set ${T}$ of \emph{simple types} in HOL is freely generated from
  a set of \emph{basic types} $BT \supseteq \{o,i\}$ using the
  function type constructor $\typearrow$. Usually, type $o$ denotes
  the (bivalent) set of Booleans, and $i$ denotes a non-empty set of
  individuals. Further base types may be added.
\end{definition}

\begin{definition}[Terms and Formulas]
  The \emph{terms} of HOL are defined as follows (where $C_{\alpha}$
  denotes typed constants and $x_{\alpha}$ typed variables distinct
  from $C_\alpha$; $\alpha,\beta,o \in {T}$):
  \begin{align*}
    s,t ::= \;\; & C_{\alpha} \mid x_{\alpha} \mid (\lambda x_{\alpha}. s_{\beta})_{\alpha \typearrow \beta} \mid 
                   (s_{\alpha \typearrow \beta}\, t_\alpha)_{\beta} 
  \end{align*}

  Complex typed HOL terms are thus constructed via
  \emph{$\lambda$-abstraction} and \emph{function application}, and
  HOL terms of type $o$ are called formulas.
\end{definition}

\sloppy As \emph{primitive logical connectives} we choose
$\neg_{o \typearrow o},\vee_{o \typearrow o \typearrow o}$ and
$\Pi_{(\alpha \typearrow o) \typearrow o}$ (for each type $\alpha$),
that is, we assume that these symbols are always contained in the
signature.
\emph{Binder notation} $\forall x_{\alpha}. s_o$ is used as an
abbreviation for
$\Pi_{(\alpha \typearrow o)\typearrow o}\lambda x_{\alpha}. s_{o}$.
Additionally, \emph{description or choice operators}
$\epsilon_{(\alpha\typearrow o)\typearrow \alpha}$ (for each type
$\alpha$) or \emph{primitive equality}
$=_{\alpha\typearrow \alpha \typearrow o}$ (for each type $\alpha$),
abbreviated as $=^\alpha$, may be added.  From the selected set of
primitive logical connectives, other logical connectives can be
introduced as abbreviations. Equality can also be defined by
exploiting Leibniz' principle, expressing that two objects are equal
if they share the same properties. Type information as well as
brackets may be omitted 
if obvious from the context or irrelevant.

We consider two terms to be \emph{equal} if the terms are the same up
to the names of bound variables (i.e., $\alpha$-conversion is handled
implicitly).

\emph{Substitution} of a term $s_{\alpha}$ for a variable $x_{\alpha}$
in a term $t_{\beta}$ is denoted by $[s/x]t$. Since we consider
$\alpha$-conversion implicitly, we assume the bound variables of $t$
are disjunct from the variables in $s$ (to avoid variable
capture).

\sloppy Prominent operations and relations on HOL terms include
\emph{$\beta\eta$-normaliza\-tion} and \emph{$\beta\eta$-equality},
\emph{$\beta$-reduction} and \emph{$\eta$-reduction}: a
\emph{$\beta$-redex} $(\lambda x. s)t$ $\beta$-reduces to $[t/x]s$; an
\emph{$\eta$-redex} $\lambda x. (s x)$ where variable $x$ is not free
in $s$, $\eta$-reduces to $s$.  It is well known, that for each simply
typed $\lambda$-term there is a unique \emph{$\beta$-normal form} and
a unique \emph{$\beta\eta$-normal form}.  Two terms $l$ and $r$ are
\emph{$\beta\eta$-equal}, denoted as $l =_{\beta\eta} r$, if their
$\beta\eta$-normal forms are identical (up to $\alpha$-conversion).
Examples of $\lambda$-conversions have been presented on
p.~\pageref{conversion} (types were omitted there).

The semantics of HOL is well understood and thoroughly documented in
the literature \cite{J6,Andrews:gmdacitt72}. The semantics of choice
for our work is Henkin \cite{Henkin50}'s general semantics. The
following sketch of standard and Henkin semantics for HOL closely
follows Benzm\"uller and Andrews \cite{J43}.

A \emph{frame} is a collection $\{D_{\alpha}\}_{\alpha \in {T}}$ of
nonempty sets, called \emph{domains}, such that $D_{o} = \{T,F\}$,
where $T$ represents truth and $F$ falsehood, $D_\itype\not=\emptyset$
and $D_\utype\not=\emptyset$ are chosen arbitrary, and
$D_{\alpha \typearrow \beta}$ are collections of total functions
mapping $D_{\alpha}$ into $D_{\beta}$.

\begin{definition}[Interpretation]
  An \emph{interpretation} is a tuple
  $\langle \{D_{\alpha}\}_{\alpha \in {T}}, I \rangle$, where
  $\{D_{\alpha}\}_{\alpha \in {T}}$ is a frame, and where function $I$
  maps each typed constant symbol $c_{\alpha}$ to an appropriate
  element of $D_{\alpha}$, which is called the \emph{denotation} of
  $c_{\alpha}$. The denotations of $\neg,\vee$ and
  $\Pi_{(\alpha \typearrow o)\typearrow o}$ are always chosen as
  usual. A variable assignment $\phi$ maps variables $X_{\alpha}$ to
  elements in $D_{\alpha}$.
\end{definition}

\begin{definition}[Henkin model]
  An interpretation is a \emph{Henkin model (general model)} if and
  only if there is a binary valuation function ${V}$, such that
  ${V}(\phi,s_{\alpha}) \in D_{\alpha}$ for each variable assignment
  $\phi$ and term $s_{\alpha}$, and the following conditions are
  satisfied for all $\phi$, variables $x_\alpha$, constants
  $C_\alpha$, and terms
  $l_{\alpha \typearrow \beta}, r_\alpha, s_\beta$
  ($\alpha,\beta\in{T}$): ${V}(\phi,x_{\alpha}) = \phi(x_\alpha)$,
  ${V}(\phi,C_{\alpha}) = I(C_{\alpha})$,
  ${V}(\phi,l_{\alpha \typearrow \beta}\, r_{\alpha}) =
  ({V}(\phi,l_{\alpha \typearrow \beta}) {V}(\phi,r_{\alpha}))$,
  and ${V}(\phi,\lambda x_{\alpha}. s_{\beta})$ represents the
  function from $D_{\alpha}$ into $D_{\beta}$ whose value for each
  argument $z \in D_{\alpha}$ is ${V}(\phi[z/x_{\alpha}], s_\beta)$,
  where $\phi[z/x_\alpha]$ is that assignment such that
  $\phi[z/x_\alpha](x_\alpha) = z$ and
  $\phi[z/x_\alpha]y_\beta = \phi y_\beta$ when
  $y_\beta \not=x_\alpha$.
\end{definition}

If an interpretation
${H} = \langle \{D_\alpha\}_{\alpha \in {T}}, I\rangle$ is an
\emph{Henkin model} the function ${V}$ is uniquely determined and
${V}(\phi,s_\alpha)\in D_\alpha$ is called the \emph{denotation} of
$s_\alpha$.

\begin{definition}[Standard model]
  ${H} = \langle \{D_\alpha\}_{\alpha \in {T}}, I\rangle$ is called a
  \emph{standard model} if and only if for all $\alpha$ and $\beta$,
  $D_{\alpha \typearrow \beta}$ is the set of all functions from
  $D_{\alpha}$ into $D_{\beta}$. Obviously each standard model is also
  a Henkin model.
\end{definition}
  
\begin{definition}[Validity]  
  A formula $c$ of HOL is \emph{valid} in a Henkin model ${H}$ if and
  only if ${V}(\phi,c) = T$ for all variable assignments $\phi$. In
  this case we write ${H}\models^{HOL} c$.  $c$ is (Henkin) valid,
  denoted as $\models^{HOL} c$, if and only if ${H}\models^{HOL} c$
  for all Henkin models ${H}$.
\end{definition}

The following theorem verifies that the logical connectives behave as
intended. The proof is straightforward.

\begin{theorem}\label{prop:trivial}
  \sloppy Let ${V}$ be the valuation function of a Henkin model
  ${H}$. The following properties hold for all variable assignments
  $\phi$, terms $s_o,t_o,l_\alpha,r_\alpha$, and variables
  $x_\alpha,w_\alpha$: ${V}(\phi,\neg s_o)=T$ if and only if
  ${V}(\phi,s_o)=F$, ${V}(\phi,s_o \vee t_o)=T$ if and only if
  ${V}(\phi,s_o)=T$ or
  ${V}(\phi,t_o)=T$, 
  ${V}(\phi,\forall x_\alpha. s_o)= {V}(\phi,\Pi_{(\alpha\typearrow o)
    \typearrow o}\lambda x_\alpha. s_o)=T$
  if and only if for all $v\in D_\alpha$ holds
  ${V}(\phi[v/w_\alpha],(\lambda x_\alpha s_o)\,w_\alpha)=T$, and
  if~$l_\alpha =_{\beta\eta} r_\alpha$ then
  ${V}(\phi,l_\alpha)={V}(\phi,r_\alpha)$.
\end{theorem}

A HOL formula $c$ that is Henkin-valid is obviously also valid in all
standard models. Thus, when a Henkin-sound theorem prover for HOL
finds a proof for $c$, then we know that $c$ is also theorem in
standard semantics.  More care has to be taken when model finders for
HOL return models or countermodels, since theoretically these models
could be non-standard. In practice this has not been an issue, since
the available model finders for HOL so far return finite models only,
and finite models in HOL are known to be standard \cite{J43}. Most importantly, however, the
returned models in \isabellehol\ can always be inspected by the user.

\section{Theories of normative reasoning covered by the \logikey\
  approach} \label{sec:deontic-logic}
 
Before explaining what theories of normative reasoning are covered by
the SSE approach, we briefly survey the area of deontic logic.  The
logics described below are all candidates for \logikey\ layer L1.

\subsection{Deontic logic}

Deontic logic~\cite{wright:51,deon:handbook,textbook18} is the field 
of logic that is concerned with normative concepts such
as obligation, permission, and prohibition. Alternatively, a deontic
logic is a formal system capturing the essential logical features of
these concepts. Typically, a deontic logic uses $Op$ to mean that ``it
is obligatory that $p$'', or ``it ought to be the case that $p$'', and
$Pp$ to mean that ``it is permitted, or permissible, that $p$''.
Deontic logic can be used for reasoning about normative multiagent
systems, i.e.,~about multiagent organizations with normative systems
in which agents can decide whether to follow the explicitly
represented norms, and the normative systems specify how, and to which
extent, agents can modify the norms. 
Normative multiagent systems need to combine normative reasoning with
agent interaction, and thus raise the challenge to relate the logic of
normative systems to aspects of agency.

There are two main paradigms in deontic logic:
the modal logic paradigm;  the norm-based paradigm.
This
classification  is
not meant to be exhaustive or exclusive. Some frameworks, like
adaptive deontic logic
\cite{DBLP:journals/sLogica/StrasserBP16,strasser}, combine the two.
We briefly
describe the two paradigms  in the next two subsections.

\subsubsection{Modal logic paradigm} \label{SDL}

Traditional (or ``standard'') deontic logic (SDL) is a normal
propositional modal logic of type KD, which means that it extends the
propositional tautologies with the axioms $K$:
$O(p\rightarrow q)\rightarrow (Op \rightarrow Oq)$ and $D$:
$\neg(Op \wedge O\neg p)$, and it is closed under the inference rules
\emph{modus ponens} $p, p\rightarrow q / q$ and \emph{generalization}
or \emph{necessitation} $p / Op$. Prohibition and permission are
defined by $Fp=O\neg p$ and $Pp=\neg O \neg p$.  SDL is an unusually
simple and elegant theory. An advantage of its modal-logical setting
is that it can easily be extended with other modalities such as
epistemic or temporal operators and modal accounts of action.

Dyadic deontic logic (DDL) introduces a conditional operator $O(p/q)$,
to be read as ``it ought to be the case that $p$, given $q$''. A
number of DDLs have been proposed to deal with so-called
contrary-to-duty (CTD) reasoning, cf.~Carmo and Jones~\cite{Carmo2002} for an overview
on this area. In brief, the CTD problem is mostly about how to
represent conditional obligation sentences dealing with norm
violation, and an example is provided Sect.~\ref{sec:GDPR}. Two
landmark DDLs are the DDL proposed by Hansson~\cite{ddl:H69}, \AA
qvist~\cite{ddl:A02,ddl:parent15} and Kratzer~\cite{krat02}, and the
one proposed by Carmo and Jones~\cite{Carmo2002,CJ13}. A notable
feature of the DDL by \AA qvist is that it also provides support for
reasoning about so-called \emph{prima facie}
obligations~\cite{ddl:Al94}. A prima facie obligation is one that
leaves room for exceptions. A limitation of the DDL by Carmo and Jones
has been pointed out by Kjos-Hanssen \cite{KH2017}.

To enable ethical agency a model of decision-making needs to be
integrated in the deontic frames. Horty's deontic STIT
logic~\cite{horty}, which combines deontic logic with a modal logic of
action, has been proposed as a starting point. The semantic condition
for the STIT-ought is a utilitarian generalization of the SDL view
that ``it ought be that $A$'' means that $A$ holds in all deontically
optimal worlds.

\subsubsection{Norm-based paradigm}

The term ``norm-based'' deontic logic has been coined by Hansen
\cite{Hansen14} to refer to a family of frameworks analysing the
deontic modalities not with reference to a set of possible worlds,
some of them being more ideal than others, but with reference to a set
of explicitly given norms.  In such a framework, the central question
is: given some input (e.g.,~a fact) and a set of explicitly given
conditional norms (a normative system), what norms apply? Thus, the
perspective is slightly different from the traditional setting,
focusing on inference patterns~\cite{pub59}.

Examples of norm-based deontic logics include the input/output (I/O)
logic of Makinson \& van der Torre
\cite{DBLP:journals/jphil/MakinsonT00}, Horty's theory of reasons
\cite{Horty12}, which is based on Reiter's default logic, and Hansen's
logic of prioritized conditional obligations~\cite{H08,Hansen14}. 

By way of illustration, we document further I/O logic. The knowledge base takes the form of a set of rules
of the form ($a$,$b$) to be read as ``if $a$ then $b$''. The key
feature of I/O logic is that it uses an operational semantics, based
on the notion of detachment, rather than a truth-functional one in
terms of truth-values and possible worlds. On the semantical side, the
meaning of the deontic concepts is given in terms of a set of
procedures, called I/O operations, yielding outputs (e.g.,
obligations) for inputs (facts). On the syntactical side, the
proof-theory is formulated as a set of inference rules manipulating
pairs of formulas rather than individual formulas.  The framework
supports functionalities that are often regarded as characteristic of
the legal domain, and thus required to enable effective legal
reasoning. Below, we list the two most elementary requirements that
can be expected of a framework, if it is to deal with the legal
domain; they are taken from Palmirani \& colleagues \cite{ruleml}.
\begin{enumerate}
\item Support for the modeling of \textit{constitutive rules}, which
  define concepts or constitute activities that cannot exist without
  such rules (e.g.,~legal definitions such as ``property"), and
  \textit{prescriptive rules}, which regulate actions by making them
  obligatory, permitted, or prohibited.
\item Implementation of
  \textit{defeasibility}~\cite{gordon95,sartor2005}; when the
  antecedent of a rule is satisfied by the facts of a case (or via
  other rules), the conclusion of the rule presumably holds, but is
  not necessarily true.
\end{enumerate}
Other norm-based frameworks provide support for the aforementioned  two
functionalities. An example of such a framework is formal argumentation \cite{Pigozzi2018}.
 The question of how to embed argumentation frameworks into HOL is ongoing and future work~\cite{C82}.

\subsection{Theories of normative reasoning implemented} \label{sec:3.2}
The following theories of normative reasoning have been
``implemented'' by utilising the SSE approach.
\begin{description}
\item{SDL:} All logics from the modal logic cube, including logic KD,
  i.e.~SDL, have meanwhile been faithfully implemented in the SSE
  approach \cite{J23}. These implementations scale for first-order and
  even higher-order extensions.
\item{DDL:} the DDL by \AA qvist \cite{ddl:A02,ddl:parent15} and the
  DDL by Carmo and Jones \cite{CJ13}: Faithful SSEs of these logics in
  \isabellehol\ are already available~\cite{J45,C71}, and most
  recently the ATP system \leoIII\ has been adapted to accept DDL as
  input \cite{steenportrayal,J51}.
\item{I/O logic \cite{DBLP:journals/jphil/MakinsonT00}:} The main
  challenge comes from the fact that the framework does not have a
  truth-functional semantics, but an operational one. First
  experiments with the SSE of the I/O-operator \emph{out}$_1$ (called
  simple-minded) and \emph{out}$_2$ (called basic) in \isabellehol\
  have been presented in the literature~\cite{J46,R65}.
\end{description}

I/O logic variants have recently been studied
\cite{ParentT17,DBLP:conf/deon/ParentT18} and some
of these variants are related to certain non-normal modal logics,
e.g.,~conditional logics with a selection function or a neighbourhood semantics. Note that the embedding of logics of the latter kind has already been studied in previous work~\cite{C37,J31}.

\section{Sample \logikey\ embedding: \AA qvist's system {\bf E} in HOL} \label{sec:example-section}

In this section, to illustrate our approach, we describe our embedding
of \AA qvist \cite{ddl:A02}'s dyadic deontic logic in HOL. The system
is called {\bf E}.

We give this example, because all the development steps foreseen at
layer L1 of \logikey\ (cf.~Sect.~\ref{sec:logikey}) have been caried
out for this system.  In particular, the embedding has been shown to
be faithfull---this is one of our key success criteria at layer
L1. Moreover, this logic has been tested against available benchmarks,
like those related to CTD reasoning. For sure, one does not need
automated theorem provers to see that the system can handle this type
of reasoning while SDL cannot. The treatment of CTD reasoning motived
DDL in the first place. However, from pen and paper results alone one
cannot conclude that an implementation of a logic will result in a
practically useful system. Such an orthogonal question must be
answered by empirical means, in particular, when first-order or
higher-order extensions, or combinations with other logics, are
required. Our SSE-based implementation of system {\bf E} scales for
such extensions; a first-order extension will later be tested with a
simple CTD example structure in Sect.~\ref{sec:GDPR}.

\subsection{Target logic: system {\bf E}}

\begin{definition} The language of {\bf E} is generated by the
  following BNF:
$$\phi::= p \mid \neg \phi \mid \phi \wedge \phi \mid \square \phi \mid \bigcirc (\phi/\phi)$$
\end{definition}
$\square\phi$ is read as ``$\phi$ is settled as true", and
$\bigcirc(\psi/\phi)$ as ``$\psi$ is obligatory,
given~$\phi$". $\bigcirc \phi$ (``$\phi$ is unconditionally
obligatory") is short for $\bigcirc(\phi/\top)$.

\medskip Traditionally so-called preference models are used as models
for the language.
\begin{definition}
  A preference model $M=(W, \succeq, V)$ is a tuple where:
  \begin{compactitem}
  \item $W$ is a (non-empty) set of possible worlds;
  \item $\succeq$ is a binary relation over $W$ ordering the worlds
    according to their betterness or comparative goodness;
    $s \succeq t$ is read as ``$s$ is at least as good as $t$";
    $\succeq$ is called a preference or betterness
    relation;\footnote{``Preference relation" is a generic name used
      in different areas like the areas of conditional logic, rational
      choice theory, non-monotonic logic, and deontic
      logic. ``Betterness relation" is the name used in deontic
      logic. This choice of name is dictated by the intuitive
      interpretation given to the preference relation.}
  \item $V$ is a valuation assigning to each propositional letter $p$
    a set of worlds, namely the set of those worlds at which $p$
    holds.
  \end{compactitem}
\end{definition}
For {\bf E}, no specific properties are assumed of the preference or
betterness relation~$\succeq$. It is known that the assumptions of
reflexivity and totalness (every worlds are pairwise comparable) do
not ``affect" the logic; cf.~Theorem \ref{th1}.

Intuitively the evaluation rule for $\bigcirc (\psi/\phi)$ puts
$\bigcirc (\psi/\phi)$ true if and only if the best $\phi$-worlds,
according to $\succeq$, are all $\psi$-worlds.  Formally:
\newcommand{\best}{\mathrm{best}_\succeq}
\newcommand{\op}{\mathrm{opt}_\succeq}
\begin{definition} [Satisfaction]
  Given a preference model $M=(W, \succeq, V)$ and a world $s\in W$,
  we define the satisfaction relation $M,s\vDash \phi$ as usual,
  except for the following two new clauses:
  \begin{compactitem}
  \item $M,s\vDash \Box\phi $ iff for all $t$ in $M$, $M,t\vDash\phi $
  \item $M,s\vDash \bigcirc(\psi/\phi)$ iff
    $\op(||\phi||)\subseteq ||\psi||$
  \end{compactitem}
  where $\Vert \phi\Vert$ is the set of worlds at which $\phi$ holds
  and $\op(||\phi||)$ is the subset of those that are optimal
  according to $\succeq$:
  \begin{flalign*}
    \op (\Vert \phi\Vert) & = \{s\in \Vert\phi\Vert\mid \forall t
    \;\;(t \vDash \phi \rightarrow s\succeq t)
  \end{flalign*}
\end{definition}
\begin{definition} [Validity] A formula $\phi$ is valid in the class
  $\mathcal{P}$ of all preference models (notation:
  $\models^{\mathcal{P}}\phi$) if and only if, for all preference
  models $M$ and all worlds $s$ in $M$, $M,s\models\phi$.
\end{definition}

\begin{definition} {\bf E} is the proof system consisting of the
  following axiom schemata and rule schemata (the names are taken from
  Parent~\cite{DBLP:journals/jphil/Parent14}):
  \begin{flalign}
    &\phi\mbox{, where $\phi$ is a tautology from PL} \tag{PL}\\
    & \square (\phi\rightarrow \psi)\rightarrow (\square \phi \rightarrow \square \psi)\label{k}\tag{K}\\
    & \square \phi\rightarrow \square\square \phi  \label{4} \tag{4}\\
    & \neg\square \phi\rightarrow \square \neg \square \phi \label{5} \tag{5}\\
    & \bigcirc(\psi\rightarrow \chi/\phi)\rightarrow (\bigcirc(\psi/\phi)\rightarrow \bigcirc(\chi/\phi))  \label{cok} \tag{COK}\\
    & \bigcirc(\phi/\phi)  \label{id} \tag{Id}\\
    & \bigcirc(\chi/(\phi\wedge \psi)) \rightarrow \bigcirc((\psi\rightarrow \chi)/\phi)  \label{sh}\tag{Sh}\\
    & 	\bigcirc(\psi/\phi) \rightarrow \square \bigcirc(\psi/\phi)  \label{abs} \tag{Abs}\\
    & \square \psi\rightarrow \bigcirc(\psi/\phi) \label{nec} \tag{Nec}\\
    & \square (\phi\leftrightarrow \psi)\rightarrow (\bigcirc(\chi/\phi) \leftrightarrow \bigcirc(\chi/\psi)  ) \label{ext} \tag{Ext}\\
    &\mbox{If } \vdash \phi \mbox{ and } \vdash\phi\rightarrow\psi  \mbox{ then } \vdash\psi \tag{MP}\\
    &\mbox{If } \vdash \phi \mbox{ then } \vdash \square\phi
    \tag{N} \label{N}
  \end{flalign}
\end{definition}

The notions of theorem and consistency are defined as usual.  \medskip

In the work of \AA qvist, cf.~\cite[pp.~179--181]{ddl:A87} and
\cite[pp.~247--249]{ddl:A02} the question of whether {\bf E} is
complete with respect to its intended modeling has been left as an
open problem; it has been recently been answered by Parent
\cite{ddl:parent15}.  See Goble \cite{gob19} for a similar
axiomatization result of Hansson's original models.

The second and third clauses in the statement of Theorem \ref{th1}
mean that the assumptions of reflexivity and totalness do not have any
impact on the logic. The fact that these two properties are idle
remains largely unnoticed in the literature.
\begin{theorem}\label{th1} {\bf E} is sound and complete with respect
  to the following three classes of preference models:
  \begin{compactenum}
  \item the class of all preference models;
  \item the class of preference models in which the betterness
    relation is required to be reflexive;
  \item the class of preference models in which the betterness
    relation is required to be total (for all $s$ and $t$, either
    $s\succeq t$ or $t\succeq s$).
  \end{compactenum}
\end{theorem}
\begin{proof} The proof can be found in Parent \cite{ddl:parent15}.
\end{proof}
\begin{theorem} The theoremhood problem in {\bf E} (``Is $\phi$ a
  theorem in {\bf E}?") is decidable.
\end{theorem}
\begin{proof}
  The proof can be found in Parent \cite{parent19}.
\end{proof}

Stronger systems may be obtained by adding further constraints on
$\succeq$, like transitivity and the so-called limit assumption, which
rules out infinite sequences of strictly better worlds.

\subsection{Embedding of {\bf E} in HOL}

The formulas of {\bf E} are identified in our SSE with certain HOL
terms (predicates) of type $i \typearrow o$, where terms of type $i$
are assumed to denote possible worlds and $o$ denotes the (bivalent)
set of Booleans.
Type $i \typearrow o$ is abbreviated as $\proptype$ in the
remainder. The HOL signature is assumed to contain the constant symbol
$R_{\itype \typearrow \proptype }$. Moreover, for each propositional
symbol $p^j$ of {\bf E}, the HOL signature must contain the
corresponding constant symbol $P^j_\tau$. Without loss of generality,
we assume that besides those symbols and the primitive logical
connectives of HOL, no other constant symbols are given in the
signature of HOL.

\begin{definition}\label{trans}
  The mapping $\lfloor \cdot \rfloor$ translates a formula $\varphi$
  of {\bf E} into a formula $\lfloor \varphi \rfloor$ of HOL of type
  $\proptype$. The mapping is defined recursively in the usual way
  \cite{J41}, except for the following two new clauses:
  \[
  \begin{array}{lcl}
    \lfloor   \Box \phi   \rfloor &=&  \Box_{ \proptype \typearrow \proptype }\,\lfloor \phi \rfloor \\
    \lfloor  \bigcirc(\psi/\phi)  \rfloor &=&  \bigcirc_{\proptype \typearrow \proptype \typearrow \proptype}\,\lfloor \psi \rfloor \lfloor \phi\rfloor    \\
  \end{array}
  \]
  where $\Box_{\proptype \typearrow \proptype }$ and
  $\bigcirc_{\proptype \typearrow \proptype \typearrow \proptype}$
  abbreviate the following formulas of HOL:
  \[
  \begin{array}{ll}
    \Box_{\proptype \typearrow \proptype } & = \lambda \phi_\proptype.  \lambda x_\itype. \forall y_{\itype}. (\phi\, y) \\
    \bigcirc_{\proptype \typearrow \proptype \typearrow \proptype} &= \lambda \psi_\proptype.  \lambda \phi_\proptype.  \lambda x_\itype. \forall w_\itype. ( (\lambda v_\itype. ( \phi \, v \wedge (\forall y_\itype. (\phi \, y \rightarrow  R_{\itype \typearrow \proptype } v \, y)) ))  \, w     \rightarrow \psi \, w ) 
  \end{array}
  \]

\end{definition}
The basic idea is to make the modal logic's possible worlds structure
explicit by introducing a distinguished predicate symbol $R$ to
represent the preference or betterness relation $\succeq$, and to
translate a formula directly according to its semantics.  For instance
$\bigcirc (\psi/\phi)$ translates into
\[
\begin{array}{ll}
  & \lambda x_i. ( ( \forall w_i. ( \phi w \wedge (\forall y_i. (\phi y \rightarrow R w y) )  \rightarrow \psi w) )
\end{array}
\]

\begin{definition}[Validity of an embedded formula] \label{val} Global
  validity ($vld$) of an embedded formula $\phi$ of {\bf E} in HOL is
  defined by the equation
  $$\text{vld}\, \lfloor \phi \rfloor=\forall z_i. \lfloor \phi
  \rfloor z$$
\end{definition}
For example, checking the global validity of $\bigcirc (\psi/\phi)$ in
{\bf E} is hence reduced to checking the validity of the formula \[
\begin{array}{ll}
  & \forall z_i. ( \phi z \wedge (\forall y_i. (\phi y \rightarrow R z y) )  \rightarrow \psi z )
\end{array}
\]
in HOL.

This definition is hidden from the user, who can construct now deontic
logic formulas involving $\bigcirc (\psi/\phi) $ and use them to
represent and prove theorems.

\subsection{Faithfulness of the embedding}

It can be shown that the embedding is faithful, in the sense given by
Thm.~\ref{faith}. Remember that the establishment of such a result is
our main success criterium at layer L1 in the \logikey\ methodology.  Intuitively, Thm.~\ref{faith} says that a
formula $\phi$ in the language of ${\bf E}$ is valid in the class of
all preference models if and only if its translation
$\lfloor \phi \rfloor$ in the language of HOL is valid in the class of
Henkin models in the sense of Def.~\ref{val}.

\begin{theorem}[Faithfulness of the embedding]\label{faith}
  \[\models^{\mathcal{P}} \phi \text{ if and only if }
  \models^\text{HOL} \text{vld}\, \lfloor \phi \rfloor\]
\end{theorem}

\begin{proof} The proof can be found in Benzm{\" u}ller {\em et al.}
  \cite{J45}. The crux of the argument consists in relating preference
  models with Henkin models in a truth-preserving way.
\end{proof}

\subsection{Encoding in Isabelle/HOL}

The practical employment of the above SSE for {\bf E} in \isabellehol\
is straightforward and can be done in a separate theory file. This
way, for a concrete application scenario, we can simply import the
embedding without dealing with any technical details. The complete
embedding is quite short (approx.~30 lines of code with line breaks)
and is displayed in Fig.~\ref{fig:e}.
\begin{figure}[tp!] \centering
  \includegraphics[width=\textwidth]{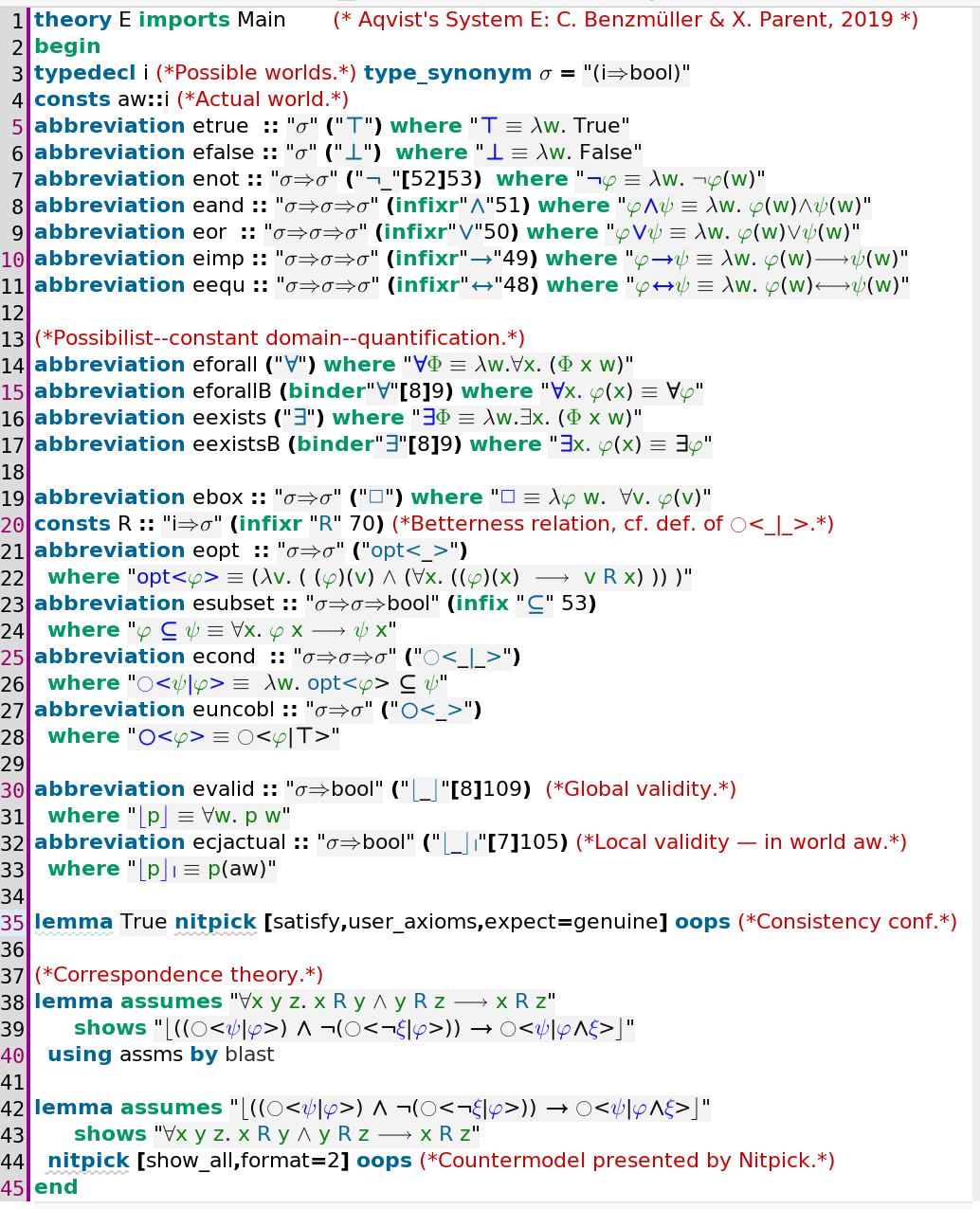}
  \caption{Embedding of the semantics of system {\bf E} in
    \isabellehol. \label{fig:e} }
\end{figure}

The embedding has been extended to include quantifiers as well. The
possibilist quantifiers (cf.~Sect.~\ref{sec:SSEapproach}) are
introduced in lines 13--17---this amounts to having a fixed domain of
individuals rather than a world-relative domain (the actualist
semantics). If needs be, actualist quantifiers can also be introduced
in a few lines of code; see the study~\cite{J41}, where analogous
definitions of both possibilist and actualist quantifiers are
presented for higher-order conditional logic. The betterness relation
$R$ is introduced as an uninterpreted constant symbol in line 20, and
the conditional obligation operator is defined in line 26. Its
definition mirrors the one given in Def.~\ref{trans}. 
The unconditional obligation operator is introduced in line 27. It is
defined in terms of its dyadic counterpart in the usual way. Last,
global and local validity (cf. Sect.~\ref{sec:SSEapproach}) are
introduced in lines 30--33.  Lines 35--40 show sample queries.  On
line 35, consistency of the embedding is confirmed. Lines 37--40
illustrate how \isabellehol\ can be used as a heuristic tool in
correspondence theory. The focus is on the assumption of transitivity
of the betterness relation. One would like to know what its
syntactical counterpart is. \isabellehol\ confirms that such an
assumption has the effect of validating the axiom Lewis \cite{ddl:L73}
called CV (line 39). \isabellehol\ also confirms that transitivity is
not equivalent to CV: the model finder Nitpick \cite{Nitpick}
integrated with \isabellehol\ finds a model validating CV in which the
betterness relation is not transitive (lines 42--44).

\section{\logikey\ tool support} \label{sec:tools}

\subsection{Support for different reasoning tasks} \label{sec:tasks}
In a nutshell, a reasoner is a tool that can perform reasoning tasks
in a given application domain.  Reasoning thereby refers to the
process of deriving or concluding information that is not explicitly
encoded in the knowledge base. 
Which information is derivable and
which is not is thereby dependent on the particular choice of logic.
The reasoning tasks that are particularly relevant in our context, for
example, include:
\begin{itemize}
\item \textit{Compliance checking}: Is the current situation,
  respectively,~an intended action by an IAS in a given context,
  compliant with a given regulation (a set of formally represented
  norms)?
\item \textit{Non-compliance analysis}: If non-compliance is the
  result of a compliance check, can the reasons be revealed and
  explained?
\item \textit{Entailment checking}: Does such-and-such obligation or
  legal interpretation follow from a given regulation?
\item \textit{Non-entailment analysis}: If entailment checking fails,
  can the reasons be revealed and explained?
\item \textit{Consistency checking}: Is a given regulation consistent?
  Is such-and-such norm, as part of a given regulation, consistent
  with this other set of norms, stemming from another regulation? Is
  such-and-such legal interpretation consistent with another one?
\item \textit{Inconsistency analysis}: If consistency checking fails,
  can a minimal set of conflicting norms be revealed and the
  inconsistency be explained?
\end{itemize}

We take consistency in its usual sense in classical logic. A set of
formulas is consistent, if there exists a model that satisfies all of
its members, and inconsistent otherwise. Thus, $\{p,q\}$ is
consistent, but $\{p,\neg p\}$ is not. Likewise,
$\{\bigcirc p, \bigcirc \neg p\}$ is consistent in a system of deontic
logic accommodating the existence of conflicts, but
$\{\bigcirc p, \neg{\bigcirc p}\}$ is not.

Consistency checking, non-compliance analysis and non-entailment
analysis are well supported by model finders, respectively
counter-model finders, while the other tasks generally require theorem
proving technology. A powerful deontic reasoner should thus ideally
provide both (counter-)model finding and theorem proving. Moreover,
intuitive proof objects and adequate presentations of (counter-)models
are desirable to enable user explanations.

\subsection{Flexibility along different axes}
While the above reasoning tasks are comparably easy to provide for
many decidable propositional fragments of deontic logics, it becomes
much less so for their quantified extensions. We are not aware of any
system, besides the work presented in this article, that still meets
all these requirements; cf.~the related work in
Sect.~\ref{sec:RelatedWork}.

Our approach addresses challenges along different axes, which are
motivated by the following observations:
\begin{description}
\item{\textit{Different logics:}} The quest for ``a most suitable
  deontic logic'' is still open, and will eventually remain so for
  quite some time. Therefore, a user should be able to choose the logic that suits his or her needs best. Our platform is flexible, and can ``host'' a wide range of deontic logics (cf.~Sect.~\ref{sec:3.2}).
  
\item{\textit{Expressivity levels:}} It is highly questionable whether
  normative knowledge and regulatory texts can often be abstracted and
  simplified to the point that pure propositional logic encodings are
  feasible and justified. For example, it can be doubted that the
  encoding of Gewirth's \emph{``Principle of Generic Consistency"},
  which we outline in Sect.~\ref{sec:PGC}, can be mechanized also in a
  propositional setting without trivialising it to a point of vacuity.
  The need for quantified deontic logics is also evidenced by related
  work such as Govindarajulu and Bringsjord's encoding of the
  \emph{``Doctrine of Double Effect"} \cite{Govindarajulu2017}.
\item{\textit{Logic combinations:}} In concrete applications normative
  aspects often meet with other challenges that can be addressed by
  suitable logic extensions, respectively, by combining a deontic
  logic with other suitable logics as required. An example is provided
  again by the encoding of Gewirth's ethical theory as outlined in
  Sect.~\ref{sec:PGC}.
\end{description}

These issues should be addressed, utilising the \logikey\ methodology,
in empirical studies in which the different choices of logics,
different expressivity levels and different logic combinations are
systematically compared and assessed within well selected application
studies. However, for such empirical work to be feasible,
implementations of the different deontic candidate logics, and its
combinations with other logics, have to be provided first, both on the
propositional level and ideally also on the first-order and
higher-order level. Moreover, it is reasonable to ensure that these
implementations remain comparable regarding the technological
foundations they are based on, since this may improve the fairness and
the significance of conceptual empirical evaluations.

Figure \ref{fig:infrastructure} well illustrates the different
components and aspects that we consider relevant in our
work. Different target logics (grey circles) and their combinations
are provided in the higher-order meta-logic of the host system. In our
case this one is either the proof assistant \isabellehol\ or the
higher-order ATP system \leoIII. In provided target logics different
ethico-legal theories can then be encoded. Concrete examples are
given in Sect.~\ref{sec:case-studies}. This set-up enables a two-way
fertilization. On the one hand, the studied theories are themselves
assessed, e.g.,~for consistency, for entailed knowledge, etc. On the
other hand the properties of the different target logics are
investigated. For example, while one of the target logics might suffer
from paradoxes in a concrete application context, another target logic
might well be sufficiently stable against paradoxes in the same
application context. An illustration of this aspect is given in
Sect.~\ref{sec:GDPR}. After arriving at consistent formalizations of
an ethico-legal theory in a suitable logic or suitable combination
of logics, empirical studies can be
conducted.\footnote{\label{foot:LeoPARD} To this end the agent-based
  LeoPARD framework~\cite{LeoPARD,C56}, which is underlying the prover
  \leoIII~\cite{Leo-III}, can be adapted and utilized, e.g.,~to embody upper ethical theories in virtual
  agents and to conduct empirical studies about the agents behaviour
  when reasoning with such upper principles in a simulated
  environment. The prover \leoIII\ is itself implemented as a
  proactive agent in the LeoPARD framework, which internally utilizes
  another copy/instance of the LeoPARD framework for the organization
  and orchestration of its distributed reasoning approach, where it
  collaborates, e.g., with other theorem provers that are modelled as
  proactive agents. In other words, \leoIII\ already \emph{is} a
  pro-active agent within an agent-based framework, and it already
  comes with deontic logic reasoning support.  The idea is now to
  populate multiple instances of such pro-active agents with selected
  normative theories within the LeoPARD framework, to initialize these
  agent societies with carefully selected tasks, and to subsequently
  study their interaction behaviour within a controlled
  environment.}

\section{\logikey\ case studies} \label{sec:case-studies}

In this section we shift our attention to the encoding of ethico-legal
theories at \logikey\ layer L2. We outline and selectively discuss two
respective case studies.

\subsection{Data protection} \label{sec:GDPR}

In our first case study the legal theory of interest is the General
Data Protection Regulation (GDPR, Regulation EU 2016/679).  It is a
regulation by which the European Parliament, the Council of the
European Union and the European Commission intend to strengthen and
unify data protection for all individuals within the European Union.
The regulation became enforceable from 25 May 2018.

First-order logic has been identified so far as a good level of
abstraction (in the case study in \ref{sec:PGC}, in contrast,
higher-order logic is used). Given respective first-order extensions
of deontic logics, the interest then is to practically assess their
correctness and reasoning performance in context, and to compare the
outcomes of such tests with our expectations. Below we illustrate how
such practical assessments can be conducted within the interactive
user-interface of \isabellehol. As an illustrating example we present
a concrete CTD structure that we revealed in the context of the GDPR;
for readers not yet familiar with CTD structures this example may also
be useful from an educational viewpoint.

The proper representation of CTD structures is a well-known problem in
the study of deontic logics.  CTD structures refer to situations in
which there is a primary obligation and, additionally, what we might
call a secondary, contrary-to-duty obligation, which comes into effect
when the primary obligation is violated.\footnote{The problem was
  first pointed out by Chisholm \cite{c63} in relation with SDL.} The
paradox arises when we try to symbolize certain intuitively consistent
sets of ordinary language sentences, sets that include at least one
CTD obligation sentence, by means of ordinary counterparts available
in various monadic deontic logics, such as SDL and similar
systems. The formal representations often turn out to be inconsistent,
in the sense that it is possible to deduce contradictions from them,
or else they might violate some other intuitively plausible condition,
for example that the sentences in the formalization should be
independent of each other. It is not the purpose of this article to
discuss in any greater depth the paradox. The interested reader should
consult, e.g., Carmo and Jones~\cite{Carmo2002}.

Here is the CTD structure we revealed in the context of the GDPR:
\vskip.7em
\noindent\fcolorbox{black}{gray!05}{\begin{minipage}{.975\textwidth}
\begin{minipage}{.95\textwidth}
\begin{enumerate}
\item \label{l1} Personal data shall be processed  lawfully (Art. 5). For example, the data subject must have given consent to the processing of his or her personal data for one or more specific purposes (Art.~6/1.a). 
\item \label{l2}  If the personal data have been processed unlawfully (none of the requirements for a lawful processing applies), the controller has the obligation to erase the personal data in question without delay (Art.~17.d, right to be forgotten).
\end{enumerate}
\end{minipage}
\end{minipage}} \vskip.7em When combined with the following a typical
CTD-structure is exhibited.  \vskip.7em
\noindent
\fcolorbox{black}{gray!05}{\begin{minipage}{.975\textwidth}
    \begin{minipage}{.95\textwidth}
      \begin{enumerate} \setcounter{enumi}{2}
      \item \label{aa} It is obligatory, e.g.,~as part of a respective
        agreement between a customer and a company, to keep the
        personal data (as relevant to the agreement) provided that it
        is processed lawfully.
      \item \label{bb} Some data in the context of such an agreement
        has been processed unlawfully.
      \end{enumerate}
    \end{minipage}
  \end{minipage}} \vskip.7em The latter information pieces are not
part of the GDPR.  \eqref{aa}~is a norm coming from another
regulation, with which the GDPR has to co-exists. \eqref{bb}~is a
factual information---it is exactly the kind of world situations the
GDPR wants to regulate.

\begin{figure}[t] 
  \includegraphics[width=\textwidth]{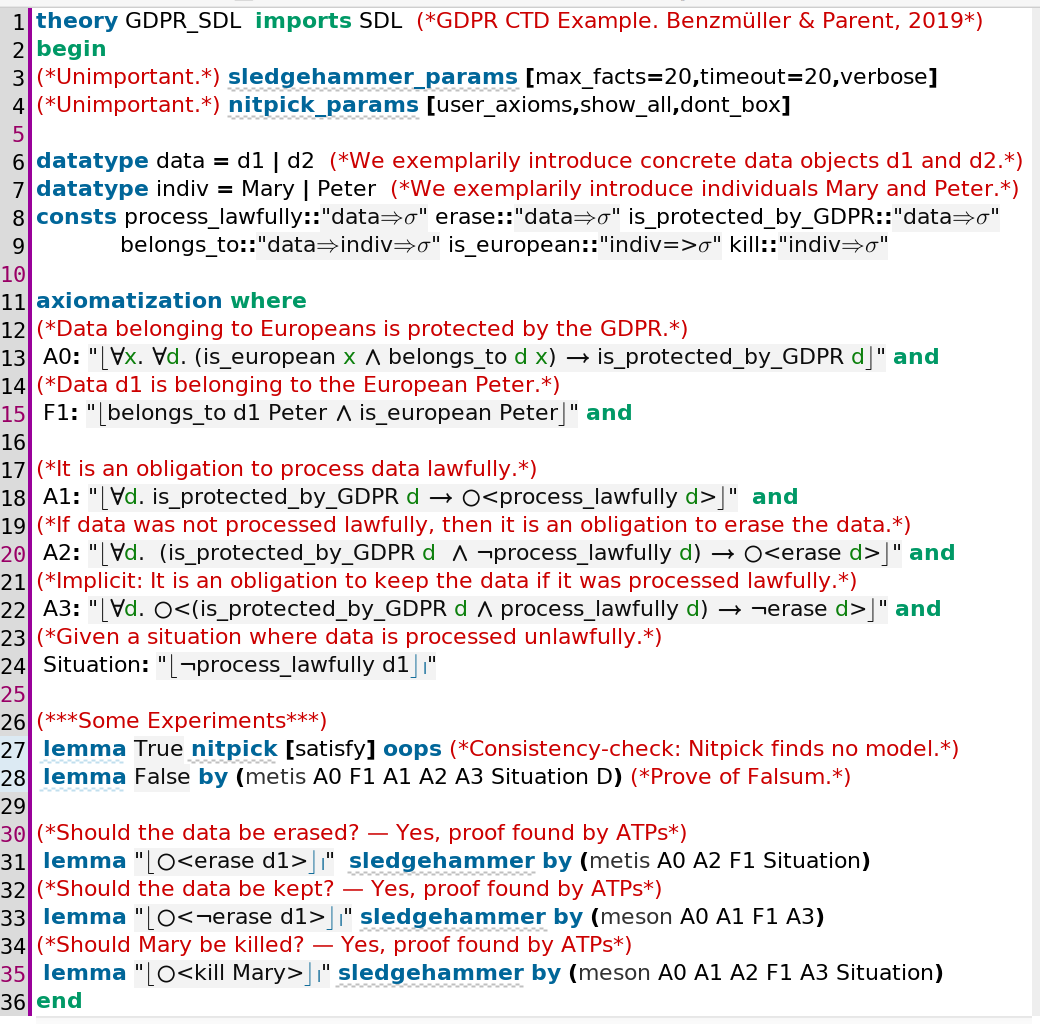}
  \caption{Failed analysis of the GDPR example in SDL. \label{GDPRinSDL}}
\end{figure}

Figure~\ref{GDPRinSDL} illustrates the problem raised by CTD
scenarios, when the inference engine is based on SDL. The knowledge
base is encoded in lines 7--25.  The relevant obligations and the
assumed situation as described by~\eqref{l1}--\eqref{bb} above are
formalized in lines 18--25.  Subsequently, three different kinds of
queries are answered by the reasoning tools integrated with
\isabellehol. The first query asks whether the encoded knowledge base
is consistent, and Nitpick answers negatively to this question in line
28. The failed attempt to compute a model is highlighted in pink. The
second query asks whether falsum is derivable, and \isabellehol's prover
\texttt{metis} returns a proof within a few milliseconds in line
29. Notice that the proof depends on the seriality axiom \texttt{D} of
SDL, which is imported from SDL's encoding in file
\texttt{SDL.thy}--that is not shown here. The query in line 36 asks
whether an arbitrarily weird and unethical conclusion such as the
obligation to kill Mary follows, and the prover answers positively to
this query.

These results are clearly not desirable, and confirm the need to use a
logic other than SDL for application scenarios in which norm violation
play a key role.

Fig.~\ref{gdpr:e} shows that our SSE based implementation of system
{\bf E} is in contrast not suffering from this effect.  The prescriptive rules of the GDPR scenario are modelled in lines 17--22:

\begin{minipage}{.975\textwidth}\centering
  \begin{lstlisting}[basicstyle=\small\ttfamily,mathescape]
    A$_1$: $\forall$d $\bigcirc(\mathsf{process}\_\mathsf{lawfully}\;\mathsf{d}/\mathsf{is}\_\mathsf{protected}\_\mathsf{by}\_\mathsf{GDPR}\;\mathsf{d}) $
    A$_2$: $\forall$d $\bigcirc(\mathsf{erase}\;\mathsf{d} /\mathsf{is}\_\mathsf{protected}\_\mathsf{by}\_\mathsf{GDPR}\;\mathsf{d}\wedge \neg \mathsf{process}\_\mathsf{lawfully}\;\mathsf{d} ) $
    A$_3$: $\forall$d $\bigcirc ( \neg \mathsf{erase}\;\mathsf{d}
    /\mathsf{is}\_\mathsf{protected}\_\mathsf{by}\_\mathsf{GDPR}\;\mathsf{d}\wedge
    \mathsf{process}\_\mathsf{lawfully}\;\mathsf{d} ) $
  \end{lstlisting}
\end{minipage}

\begin{figure}[t]
  \includegraphics[width=\textwidth]{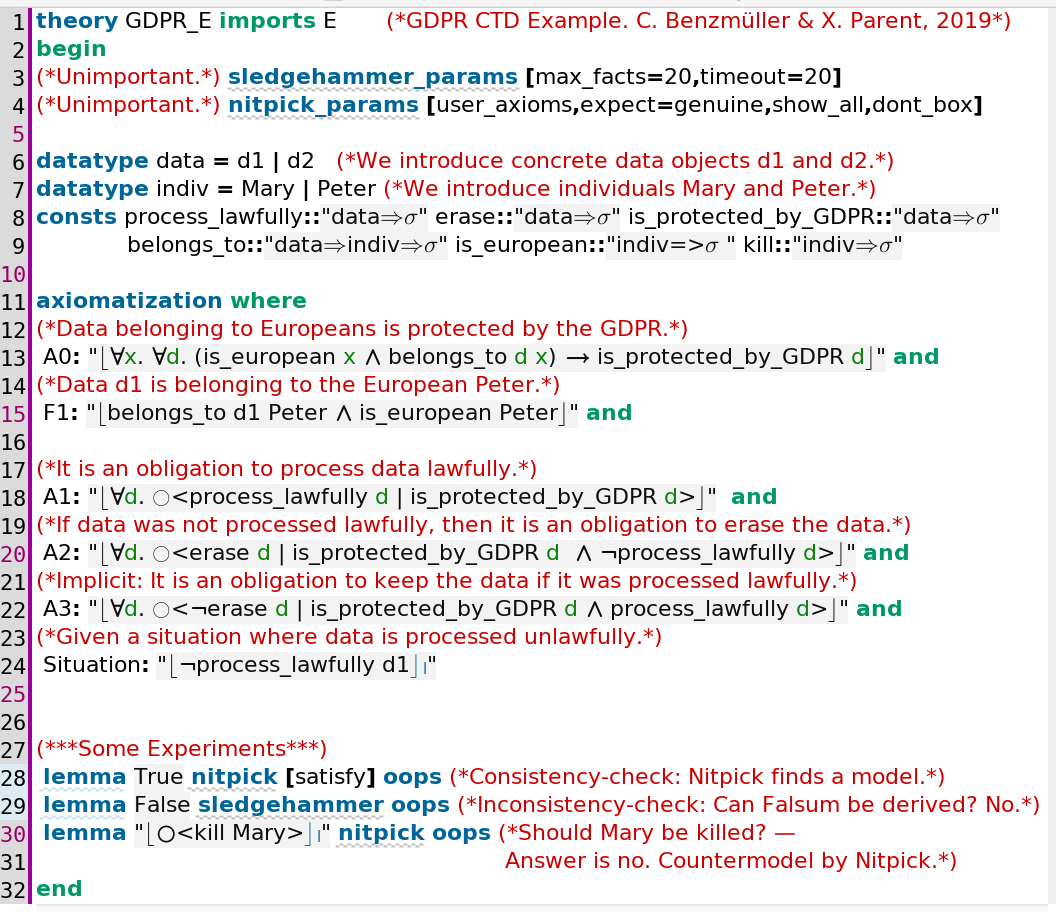}
  \caption{Successful analysis of the GDPR example scenario in system
    $\mathbf{E}$.} \label{gdpr:e}
\end{figure} 

\noindent The current situation, in which we have that Peter's
personal data $\mathsf{d1}$ are not processed lawfully, is defined in
line 24:

\begin{minipage}{.975\textwidth}\centering
  \begin{lstlisting}[basicstyle=\small\ttfamily,mathescape]
    Situation:  $\neg \mathsf{process}\_\mathsf{lawfully}\;\mathsf{d1} $
  \end{lstlisting}
\end{minipage}
The three same queries as before are run, but this time with
success. In line 28 we are told by Nitpick that the knowledge base has
a model. The computed model can be inspected in full detail a separate
window. In line 29 we are told that falsum is no longer derivable and
the theorem provers integrated with sledghammer terminate with a
time-out. In line 30 we are told that the obligation to kill Mary no
longer follows.

These practical experiments with theorem provers and model finders
demonstrate that a machine can indeed reason about norm violation in a
first-order extension of DDL at run-time, resp.~on the fly.  Future
research will study the practical deployment of the \logikey\ flexible
reasoning technology, for example, for ethico-legal governance of
autonomous cars, smart homes, social robots, etc.  Given these
ambitious goals, it is relevant to confirm \emph{early-on} that the
\logikey\ solution, without further ado, can already handle simple
examples as discussed here.\footnote{\label{footnote:stit}For example,
  Meder, in his MSc thesis~\cite{MederMasters}, worked out an SSE for
  a variant of STIT-logic put forth by Lorini \cite{lori2013}, called
  T-STIT, for temporal STIT. It was found out that model and
  (counter-)model finding wasn't responsive for analogous
  experiments. The reason, however, was quickly revealed: the
  T-STIT-logic, by nature, requires infinite models, while the
  (counter-)model finding tools in \isabellehol\ so far explore finite
  model structures only.}  This, of course, provides only a relevant
first assessment step.  The development of respective libraries of
increasingly complex benchmarks, see also Sect.~\ref{sec:PGC}, is
therefore an important objective in our future work.

\subsection{Gewirth's Principle of Generic
  Consistency} \label{sec:PGC}

\begin{figure}[tp]
  \centering \fcolorbox{black}{gray!05}{
    \begin{minipage}{.97\textwidth}
      Alan Gewirth's \emph{``Principle of Generic Consistency} (PGC)
      \cite{GewirthRM,Beyleveld}, constitutes, loosely speaking, an
      emendation of the \emph{Golden Rule}, i.e., the principle of
      treating others as one's self would wish to be treated. Adopting
      an agent perspective, the PGC expresses and deductively
      justifies a related upper moral principle, according to which
      any intelligent agent, by virtue of its self-understanding as an
      agent, is rationally committed to asserting that it has rights
      to freedom and well-being, and that all other agents have those
      same rights. The main steps of the argument are
      (cf.~\cite{Beyleveld} and~\cite{C76}):
      \begin{description}[itemsep=0pt,leftmargin=10pt]
      \item[(1, premise)] I act voluntarily for some purpose E, i.e.,
        I am a prospective purposive agent (PPA).
      \item[(2, derived)] E is (subjectively) good (i.e.~I value E
        proactively).
      \item[(3, derived)] My freedom and well-being (FWB) are
        generically necessary conditions of my agency (I need them to
        achieve any purpose whatsoever).
      \item[(4, derived)] My FWB are necessary goods (at least for
        me).
      \item[(5, derived)] I have (maybe nobody else) a claim right to
        my FWB.
      \item[(13, final conclusion)] Every PPA has a claim right to
        their FWB.
      \end{description}
      Formalization of the argument is challenging; it faces complex
      linguistic expressions such as alethic and deontic modalities,
      quantification and indexicals.  The solution of Fuenmayor and
      Benzm\"uller \cite{C76,C77}, cf.~Fig.~\ref{fig:GewirthProof2}
      for an excerpt, is based on the following modeling decisions:
      \texttt{FWB} and \texttt{Good} are introduced as unary
      uninterpreted predicate symbols; further uninterpreted relation
      symbols are added: \texttt{ActsOnPurpose},
      \texttt{InterferesWith} (both 2-ary) and
      \texttt{NeedsForPurpose} (3-ary).  $\texttt{PPA}\, a$ is defined
      as $\exists E.\, \texttt{ActsOnPurpose}\, a\, E$; an additional
      axiom postulates that being a \texttt{PPA} is
      identity-constitutive for any individual:
      $\lfloor \forall a.  \texttt{PPA} a \rightarrow {\Box}^D
      (\texttt{PPA} a)\rfloor^D$.
      $\lfloor \varphi\rfloor^D$ in there models indexical validity of
      $\varphi$; it is defined following Kaplan's logic of
      demonstratives \cite{Kaplan1979} as being true in all contexts.
      $\texttt{RightTo}\, a \, \varphi$ is defined as
      $\textbf{O}_i(\forall b. \neg \texttt{InterferesWith}\, b\,
      (\varphi\ a))$;
      this captures that an individual $a$ has a (claim) right to some
      property $\varphi$ iff it is obligatory that every (other)
      individual $b$ does not interfere with the state of affairs
      $(\varphi\, a)$.  $\textbf{O}_i$ is defined as the \emph{ideal}
      obligation operator from Carmo and Jones \cite{Carmo2002} (their
      actual obligation operator $\textbf{O}_a$ could be used as
      well).

      The meaning of the uninterpreted constant symbols is constrained
      by adding further axioms: e.g., axioms that interrelate the
      concept of goodness with the concept of agency, or an axiom
      $\lfloor\forall P. \forall a.\, \texttt{NeedsForPurpose}\, a\,
      \texttt{FWB}\, P\rfloor^D$
      that expresses that \texttt{FWB} is always required in order to
      act on any purpose, whereas \texttt{FWB} is postulated to be a
      contingent property
      ($\lfloor\forall a. \Diamond_p\, \texttt{FWB}\, a\ \wedge
      \Diamond_p\, \neg \texttt{FWB}\, a\rfloor^D$).

      Note that both first-order and higher-order quantifiers are
      required; cf.~the ``$\forall P. \forall a.$''-prefix in the
      previous axiom, where $P$ ranges over properties and $a$ over
      individuals.
    \end{minipage}
  }
  \caption{Gewirth's \emph{Priciple of Generic Conistency}, and its
    formal encoding \cite{C76}, in a nutshell.}
  \label{fig:GewirthArgument}
\end{figure}

In a second case study our framework has been adapted and utilized for
the exemplary mechanization and assessment \cite{C76,C77} of Alan
Gewirth's \emph{Principle of Generic Consistency} (PGC)
\cite{GewirthRM,Beyleveld} on the computer; the PGC is an ethical
theory that can be understood as an emendation of the well known
golden rule.  The formalization of the PGC, which is summarized below,
demonstrates that intuitive encodings of ambitious ethical theories
and their mechanization, resp.~automation, on the computer are no
longer antipodes.

Some further information on the PGC is given in
Fig.~\ref{fig:GewirthArgument}, including a depiction of the modeling
decisions taken for its formalization.
The self-contained \isabellehol\ sources of the formalization are
available in the Archive of Formal Proofs \cite{GewirthProofAFP},\footnote{See also the \emph{Data in brief} article \cite{J53} associated with this article.} and
an excerpt of the derivation and verification of the main inference
steps of the argument is presented in Fig.~\ref{fig:GewirthProof2}.
\begin{figure}[ht!] \centering
  \includegraphics[width=\textwidth]{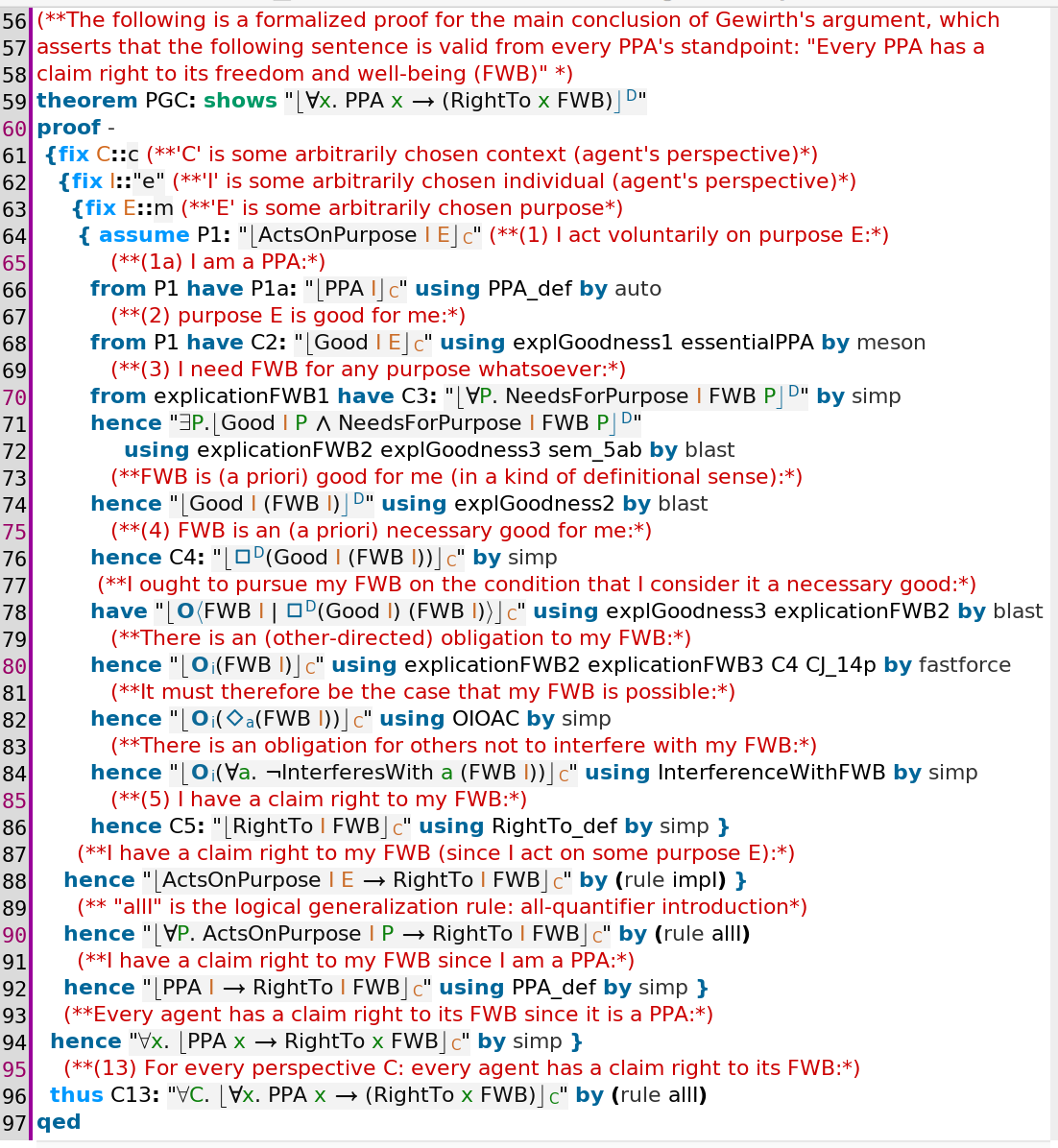}
  \caption{Mechanization of Gewirth's argument in \isabellehol,
    cf.~\cite{C76}. \label{fig:GewirthProof2}}
\end{figure}

The argument used by Gewirth to derive the PGC is by no means trivial
and has stirred much controversy in legal and moral philosophy during
the last decades. It has also been discussed in political philosophy
as an argument for the \textit{a priori} necessity of human
rights. Perhaps more relevant for us, the PGC has lately been proposed
as a means to bound the impact of artificial general intelligence
\cite{Kornai}.

Gewirth's PGC is taken here as an illustrative showcase to exemplarily
assess its logical validity with our normative reasoning machinery,
and the success of Fuenmayor's and Benzm\"uller's work \cite{C76,C77}
provides evidence for the claims made in this article.  It is the
first time that the PGC has been formalized and assessed on the
computer at such a level of detail, i.e.,~without trivializing it by
abstraction means, and while upholding intuitive formula
representations and user-interaction means. As a side-effect, the ATP
systems integrated with \isabellehol\ have helped in this work to
reveal and fix some (minor) issues in Gewirth's argument.

To encode the PGC an extended embedding of the DDL by Carmo and Jones in HOL was employed.  Conditional, ideal and actual
obligation has been
combined, 
among others, with further modalities, with Kaplan's notion of
context-dependent logical validity, and with both first-order and
higher-order quantifiers.\footnote{Kaplan's context-sensitive logic of
  demonstratives covers the dimension of agency, among others. Future
  work includes the study of alternative modelings, utilising,
  e.g.,~the modal logics of agency or STIT logic. STIT logic, however,
  gives rise to interesting problems, including infinite models as
  already noted. The infinite character is fixed in the axioms.} %

What can be seen in Fig.~\ref{fig:GewirthProof2} is that readable
formula presentations are supported, which in turn enable quite
intuitive user-interactions in combination with proof automation
means. It is demonstrated that the interactive assessment of Gewirth's
argument is supported at an adequate level of granularity: the proof
steps as presented in Fig.~\ref{fig:GewirthProof2} correspond directly
to actual steps in Gewirth's argument and further details in the
verification process can be delegated to automated reasoning tools
provided in \isabellehol.

Proof automation can be pushed further, but this has not been the
focus so far in this work. The result of this work is a formally
verified derivation of the statement \emph{``Every prospective
  purposive agent has a claim right to their free will and
  well-being''} from Gewirth's elementary premises, axioms and
definitions in the theorem prover \isabellehol\ (at \logikey\ layer L2,
cf.~Fig.~\ref{fig:metho}). The formalization of this ethical theory in
combination with the mechanization and automation of its underlying
non-trivial, logic combination (at \logikey\ layer L1) enables further
experiments with the theory and provides an entry for devising
applications of it (at \logikey\ layer L3). One possibility in this
regard is the population of this theory within virtual agents, so that
the behaviour of each agent is ultimately constraint, if no other
regulatory rules apply, by this very general moral principle;
cf.~Footnote~\ref{foot:LeoPARD}.

\section{Logic \& knowledge design and engineering methodology}
\label{sec:logikey}

Now that we have explained and illustrated various ingredients and
aspects of \logikey, we will, in this section, present and discuss the
\logikey\ logic \& knowledge design and engineering methodology
methodology in more detail. Remember that \logikey\ distinguishes
between the following three layers; they were visualized earlier in
Fig.~\ref{fig:metho}.

\begin{enumerate}[{L}1]
\item \emph{Logics and logic combinations.} Logics and logic
  combinations are modelled, automated and assessed at layer L1. The
  automation of CTD compliant normative reasoning is thereby of
  particular interest, where logical explosion is avoided even in
  situations when an intelligent autonomous system (IAS) accidentally
  or deliberately violates some of its obligations. We thus prefer a
  proper intra-logical handling of the CTD challenge in our work over
  tackling this challenge, e.g., with extra-logical and ad hoc
  means. Sect.~\ref{sec:deontic-logic} presented respective examples
  of CTD compliant logics---the dyadic deontic logic (DDL) of Carmo
  and Jones~\cite{Carmo2002} and \AA qvist's system
  \textbf{E}~\cite{ddl:A02,ddl:parent15}---and
  Sect.~\ref{sec:example-section} explained how system \textbf{E} is
  modelled and automated adopting the \logikey\ methodology.
\item \emph{Ethico-legal domain theories.}  At layer L2 concrete
 ethico-legal theories are modelled, automated and assessed ---
  exploiting the logic and logic combinations from layer L1. Question
  answering, consistency and compliance checking support is provided.
  Our particular interest is in ethico-legal  theories suited to
  govern the behaviour of IASs. Examples for layer L2 developments
  were presented in Sect.~\ref{sec:case-studies}, and tool support,
  for L2 and the other layers, was discussed in Sect.~\ref{sec:tools}.
\item \emph{Applications.}  Layer L3 addresses the deployment of the
  ethical and legal domain reasoning competencies, e.g., for regulating
  the behaviour of IASs.  This requires the implantation of L1- and
  L2-reasoning capabilities within a suitably designed governance
  component for the IAS to be regulated.
\end{enumerate}

Concrete engineering steps for layers L1, L2 and L3 are presented
below, and pointers to directly related work and illustrating examples
are provided. The \logikey\ methodology has been developed, refined
and tested in prior research projects and also in lecture courses at
Freie Universit\"at Berlin and at University of Luxembourg. We have
thus gained experience that \logikey\ well supports both research and
education~\cite{C58} at the depicted layers.

While each layer in the methodology can serve as an entry point, we
here start our systematic exhibition at layer L1. In concrete research
and education projects, one may alternatively start from layers L2 or
L3; the formalization of Gewirth's PGC, for example, started at layer
L2.  Entry at layers L3 or L2 will increasingly become the default,
when our library of reusable logic encodings at layer L1 grows, so
that extensive engineering work at layer L1 can largely be avoided in
future applications.

\subsection{Layer {L1} -- Logics and logic
  combinations} \label{sec:L1}

\paragraph{Logic: Select a logic or logic combination}
Logic examples include DDL or system \textbf{E},
cf.~Sect.~\ref{sec:deontic-logic}. The need for a non-trivial logic
combination was, e.g., identified in the case study on Gewirth's
PGC~\cite{C76}; cf.~Sect.~\ref{sec:PGC} and
Fig.~\ref{fig:GewirthArgument}. As explained in detail by Fuenmayor
and Benzm\"uller~\cite{C76}, in this work a higher-order extension of
DDL was combined with relevant parts of Kaplan's logic of
demonstratives (LD)~\cite{Kaplan1979,Kaplan1989}.  Moreover, a
combination of multi-epistemic logic with an operator for common
knowledge was used in an automation of the Wise Men
Puzzle~\cite{J41,J44}.

\paragraph{Semantics: Select a semantics for the chosen logic or logic
  combinations} This step is preparing the automation of the logic or
logic combination, using the SSE approach, in the next step.  Suitable
semantics definitions for DDL and system \textbf{E} are given in the
mentioned original articles. To arrive at a semantics for the logic
combination used in the PGC formalization, DDL's original semantics
was extended and combined, utilising a product construction \cite[\S
1.2.2]{LogicCombining}, with Kaplan's semantics for LD.  In the work
on the Wise Men Puzzle, a logic fusion~\cite[\S 1.2.1]{LogicCombining}
was used, and the semantics for multi-epistemic logic and common
knowledge were adopted from Sergot \cite{Sergot08}.
 
\paragraph{Automate: Automate the selected logics and logic
  combinations in HOL} In this step a selected logic or logic
combination is shallowly embedded by directly encoding its semantics
in meta-logic HOL. This is practically supported in theorem provers
and proof assistant systems for HOL, such as \isabellehol\
or \leoIII.  SSEs thus constitute our preferred
solution to the automation challenge, since they enable the reuse of
already existing reasoning machinery.\footnote{However, also deep
  logic embeddings in HOL can be utilized, and even combinations of
  both techniques are supported. The \emph{abstraction layers} used by
  Kirchner, Benzm\"uller and Zalta \cite{J47}, for example, support a
  deep embedding of a proof calculus for higher-order
  hyper-intensional (relational) logic S5 in HOL on top of an SSE for
  the same logic in HOL; both reasoning layers can then been exploited
  for proof automation and mutual verification; cf.~also Footnote
  \ref{foot:deep} and Sect.~\ref{sec:RelatedWork}.}
 
\paragraph{Assess: Empirically assess the SSEs with model finders and
  theorem provers} The consistency of an SSE can be verified with
model finders, and ATP systems can be employed to show that the
original axioms and proof rules of an embedded logic can be derived
from its semantics. Respective examples have been presented in the
literature~\cite{C77,J47}. If problems are revealed in this
assessment, modifications in previous steps maybe required.
 
\paragraph{Faithfulness: Prove the faithfulness of the SSEs} Part of
this challenge is computationally addressed already in the previous
step. In addition, explicit soundness and completeness proofs should
be provided using pen and paper methods.  An example is
Thm.~\ref{faith} in Sect.~\ref{sec:example-section}, which proves the
faithfulness of the SSE of system \textbf{E} in HOL. If faithfulness
cannot be established, modifications in the first three steps maybe
required.

\paragraph{Implications: Explore the implications of the embedded
  logics and logic combinations} What theorems are implied by an SSE?
Are the expectations matched?  For the SSE of the combination
higher-order DDL with Kaplan's LD in HOL it has, e.g., been
checked~\cite{C77}, whether model finders still report countermodels
to propositional and first-order variants of the Chisholm
paradox~\cite{c63}.  If the investigated implications do not match the
expectations, modifications in the first three steps maybe required.

\paragraph{Benchmarks: Test the logic automation against benchmarks}
Benchmark tests not only rank \logikey\ proof automations against
competitor approaches, they also provide further evidence for the
practical soundness and robustness of the implemented technology; this
is most relevant since pen and paper faithfulness proofs provide no
guarantee that an implementation is free of bugs.\footnote{Respective
  comparisons of SSE provers for first-order modal logics with
  competitor systems have been presented in the
  literature~\cite{C62,C34}. Since we are not aware of any
  implementation of first-order or higher-order DDLs besides our own
  work, we cannot yet conduct analogous evaluation studies for our
  SSE-based DDL provers.}

\paragraph{Contribute: Contribute to the built-up and extension of
  benchmark suites}

\subsection{Layer L2 -- Domain Theories} \label{sec:L2}

\paragraph{Select an ethico-legal domain theory of interest} An
example is Gewirth's PGC from~Sect.~\ref{sec:PGC}. Another example is
the German road traffic act, which we have started to work on in
student projects as part of our lecture courses. Alternatively, or in
addition, one might be interested in designing, from scratch, new sets
of ethical rules to govern the behaviour of autonomous cars.

\paragraph{Analyse the ethico-legal  domain theory} Mutually
related aspects are addressed in this step, and requirements for the
logics and logic combinations imported from layer L1 are identified.
\begin{enumerate}
        
\item \emph{Determine a suitable level of abstraction.}  Can relevant
  notions and concepts be suitably abstracted, e.g., to a purely
  propositional level (as often done in toy examples in AI), or would
  that cause an oversimplification regarding the intended applications
  at layer L3?
   
\item \emph{Identify basic notions and concepts.}  What are the most
  essential concepts addressed and utilized in a given domain theory.
  Which basic entities need to be explicitly referred to, and which
  properties, relations and functions for such entities must be
  explicitly modelled?  For example, notions identified as relevant
  for Gewirth's PGC, cf.~Fig.~\ref{fig:GewirthArgument}, include the
  relations \texttt{ActsOnPurpose}, \texttt{NeedsForPurpose} and
  \texttt{InterferesWith}, the predicates \texttt{Good} (for Goodness)
  and \texttt{FWB} (free will and well-being), and the defined terms
  \texttt{RightTo} (has right to) and \texttt{PPA} (is a prospective
  purposive agent). As a result of this analysis, a \emph{signature}
  is obtained together with a set of associated foundational axioms
  and definitions.
    
  In this step dependencies on very generic notions, such as
  mereological terms, may be revealed whose precise meanings are left
  implicit. For the formalization of such notions one may consult
  other sources, including existing upper ontologies.\footnote{Upper
    ontologies formally define very generic terms that are shared
    across several domains. Import from external upper ontologies
    requires some conceptual and logical fit which often is not
    given.}
\item \emph{Identify notions of quantification.}  Domain theories may
  contain universal and/or existential statements that cannot or
  should not be abstracted away. Careful assessment of the precise
  characteristics of each of the identified quantifiers is then
  advisable. In particular, when quantifiers interact with linguistic
  modalities, see below, the question arises whether, e.g., the Barcan
  formulas \cite{Barcan46} should hold or not. Different kinds of
  quantifiers may thus have to be provided at layer L1.
\item \emph{Identify linguistic modalities.}  Ethico-legal domain
  theories are challenged, in particular, by deontic modalities (e.g.,
  "an entity is \emph{permitted/obliged} to \ldots"), and they may
  occur in combination with other modalities. In Gewirth's PGC, for
  example, deontic and alethic modalities are relevant. Notions of
  time or space are further examples that frequently need to be
  addressed. Combinations of modalities may thus have to be provided
  at layer L1.
\end{enumerate}

\paragraph{Determine a suitable logic or logic combination}
The previous step identifies essential logical requirements for the
formalization task at hand. Based on these requirements a suitable
base logic or logic combination must be determined or devised and
imported from layer L1; if not yet provided, further work at layer L1
is required.

\paragraph{Formalize the ethico-legal  domain theory} During the
formalization process regular sanity checks, e.g., for in-/consistency
or logical entailment, with ATP systems and model finders are advisable,
cf.~Sect.~\ref{sec:tasks}. This serves two different purposes: early
detection of misconceptions and errors in the formalization process,
and early testing of the proof automation performance.  If the
reasoning performance is weak early on, then countermeasures maybe
taken, for example, by considering alternative choices in the previous
steps.\footnote{Remember Footnote~\ref{footnote:stit}, where
  respective issues for T-STIT-logic where identified by such
  experiments.}

\paragraph{Explore theory implications and check whether expectations are matched} 
If the computationally explored and assessed implications of the
formalized domain theory are not matching the expectations,
modifications in one of the previous steps are
required.\footnote{Respective experiments have been conducted, e.g.,
  in the context of the formalisation of G\"odel's ontological
  argument. In these experiments it has been confirmed with automated
  reasoning tools that monotheism and the modal collapse (expressing
  ``there are no contingent truths'', ``there is no free will'') were
  implied by G\"odel's theory; cf.~\cite{C55}. Both implications might
  not be in-line with our expectations or intentions, and the modal
  collapse has in fact motivated further emendations of G\"odel's
  theory and of the utilized foundational logic.}

\paragraph{Contribute: Contribute to the built-up of benchmark suites
  for domain theories}

\subsection{Layer L3 -- Applications} \label{sec:L3} Layer L3 deploys
the ethico-legal domain theories from layer L2 in practical
applications, e.g., to regulate the behaviour of an IAS.

\paragraph{Select an application scenario and define the objectives}
Already mentioned examples include the ethico-legal governance of
autonomous cars and ethico-legal reasoner in the smart home example.

\paragraph{Ethical governor component} 
A suitable explicit ethical governor architecture must be selected and
provided. This step connects with recent research area on governing
architectures for intelligent systems
\cite{GovindarajuluB15,Arkin12,Arkin09}.

 \paragraph{Populate the governor component with the ethico-legal domain theory and reasoner}
 Select the ethico-legal domain theory to be employed; if not yet
 formalized and automated, first consult layer L2.  Otherwise
 integrate the ethico-legal domain reasoner obtained from layer L2
 with the governor component and perform offline tests. Does it
 respond to example queries as expected and with reasonable
 performance? In particular, can it check for compliance of a
 considered action in an IAS wrt.~the formalized ethico-legal domain
 theory?

 \paragraph{Properly test, assess and demonstrate the system in
   practice}

 \subsection{Note on open-texture and concept explication}

 A goal of \logikey\ at layer L2 is to support, among others, the
 systematic exploration, formalization and automation of \emph{new}
 regulatory theories for IASs.  Ideally, such a development is
 conducted by an expert team comprised of logicians, legal and ethical
 experts and practitioners representing the addressed application
 domain. Alternatively, such a theory formalization process may start
 from existing legislation. A known challenge in the latter case
 concerns the open-texture of informal legal texts; e.g., relevant
 concepts may probably be specified rather vaguely and deliberately
 left open for interpretation in context.  In prior work the
 open-texture challenge occurred also in other contexts, e.g., in the
 analysis philosophical arguments. In a recent reconstruction and
 verification of Lowe's ontological argument \cite{Lowe}, for example,
 underspecified notions such as \emph{necessary being} and
 \emph{concrete being} had to be suitably interpreted in context to
 finally arrive at a verified formalisation of the argument;
 cf.~Fuenmayor and Benzm{\"u}ller~\cite{J38}, where a respective computer-supported, iterative
 interpretation process is presented and explained.  The \logikey\
 methodology does not eliminate the open-texture challenge. However,
 it provides suitable means to address it in an interaction between
 human experts and computational tools. The idea is to computationally
 explore suitable explications or emendations of vague concepts for a
 given application context. The logical plurality, flexibility and
 expressivity supported in \logikey\ thereby constitutes a distinctive
 virtue that is, as far as we are aware of, unmatched in related work.
 Moreover, ongoing work on \emph{computational
   hermeneutics}~\cite{B19} aims at reducing the need for user
 interaction in future applications by automating the exploration of
 choice points in concept explication and beyond.

 \section{Related work} \label{sec:RelatedWork}

 Relevant own related work has already been mentioned in the previous
 sections; the referenced works contain technical details and evidence
 that we cannot address in full depth in this article. Further
 references to own work and that of others are provided below. Many of
 those contain illustrating examples and further details that may well
 benefit researchers or students interested in adopting the \logikey\
 methodology in their own work. They also provide useful information
 on various intellectual roots of the research presented in this
 article.

 \subsection{Machine ethics and deontic logic}

 The questions how transparency, explainability and verifiability can
 best be achieved in future intelligent systems and whether bottom-up
 or top-down architectures should be preferred are discussed in a
 range of related articles; cf.,
 e.g.,~\cite{Malle2016,dignum18:_special_issue,Winfield2018,ijcai2017-655,Scheutz2017,DBLP:journals/irob/AndersonA15,Wallach,DBLP:journals/ras/DennisFSW16}
 and the references therein. For example, Dennis {\em et
   al.}~\cite{DBLP:journals/ras/DennisFSW16} make a compelling case
 for the use of so-called formal verification---a well-established
 technique for proving correctness of computer systems---in the area
 of machine ethics. The idea is to use formal methods to check whether
 an autonomous system (e.g., an unmanned civilian aircraft) is
 compliant with some specific ethico-legal rules (e.g., the Rules of
 the Air) when making decisions.  An ethical rule is represented as a
 formula of the form ``$\mathrm{do(a)}\Rightarrow_c\neg E\phi$",
 denoting that doing action a counts as a violation of ethical
 principle $\phi$. However, they do not specify in full the syntax and
 semantics of their operator $E$.  It may be valuable to further
 explore the relationship between this work and the approach outlined
 in the present article.  As the authors observe, on p.~6 of their
 article, the $E$ modal operator resembles the obligation operator
 $\bigcirc$ used in deontic logic.

 Further related work includes a range of implemented theorem proving
 systems.  A lean but powerful connection-based theorem prover for
 first-order modal logics, covering also SDL, has been developed by
 Otten~\cite{DBLP:conf/tableaux/Otten17}. A tableaux-based
 propositional reasoner is employed in the work of Furbach and
 Schon~\cite{DBLP:conf/tableaux/Otten17,DBLP:conf/birthday/FurbachS14}
 and first-order resolution methods for modal logics have been
 contributed by Schmidt and Hustadt
 \cite{DBLP:conf/birthday/SchmidtH13}.  Further related work includes
 a reasoner for propositional defeasible modal logic by Governatori
 and his team~\cite{DBLP:journals/ijswis/KontopoulosBGA11}. Their
 reasoner supports defeasible reasoning, but it is less flexible than
 ours, because it does not allow the user to easily switch between
 different systems of normative reasoning and explore their
 properties.  A reasoner for expressive contextual deontic reasoning
 was proposed by Bringsjord {\em et
   al.}~\cite{DBLP:conf/isaim/BringsjordGMS18}.  Pereira and
 Saptawijaya
 \cite{DBLP:series/sapere/PereiraS16,DBLP:journals/igpl/SaptawijayaP16,pereira17:_count}
 present a solution that implements deontic and counterfactual
 reasoning in Prolog.

 We are not aware of any attempts to automate, within a single
 framework, such a wide portfolio of CTD resistant propositional,
 first-order and higher-order deontic logics as we report it in this
 article.  Note that in addition to the features of the above related
 systems our solution also supports intuitive user-interaction and
 most flexible logic combinations.

 The SSE approach has also been implemented in the \leoIII\ theorem
 prover, so that the prover now provides native language support for a
 wide range of modal logics and for DDL \cite{steenportrayal,C62}. A
 recent, independent study \cite{GRUNGE2019} shows that \leoIII, due
 to its wide range of directly supported logics, has become the most
 powerful and most widely applicable ATP system existent to
 date.\footnote{The assessment has included various variants of
   classical first-order and higher-order logic benchmark
   problems. First-order and higher-order deontic logics and other
   non-classical logics were still excluded though. Their inclusion
   would clearly further benefit the \leoIII\ prover.}

 The flexibility of the SSE approach has been exploited and
 demonstrated in particular in the case study on Gewirth's PGC that we
 have presented in Sect.~\ref{sec:PGC}.  Related work on the
 mechanization of ambitious ethical theories includes the already
 mentioned automation of the \emph{``Doctrine of Double Effect"} by
 Govindarajulu and Bringsjord \cite{Govindarajulu2017}.

\subsection{Universal logical reasoning}
Related experiments, at \logikey\ layers L1 and L2, with the SSE
approach have been conducted in metaphysics \cite{C55,C40}. An initial
focus thereby has been on computer-supported assessments of rational
arguments, in particular, of modern, modal logic variants of the
{ontological argument for the existence of God}. In the course of
these experiments, in which the SSE approach was applied for
automating different variants of higher-order quantified modal logics,
the theorem prover \leoII\ even detected an previously unnoticed
inconsistency in G\"odel's \cite{GoedelNotes} modal variant of the
ontological argument, while the soundness of the emended variant by
Scott \cite{ScottNotes} was confirmed and all argument steps were
verified. Further modern variants of the ontological argument have
subsequently been studied with the approach, and theorem provers have
even contributed to the clarification of an unsettled philosophical
dispute \cite{J32}. The good performance of the SSE approach in
previous work has been a core motivation for the new application
direction addressed here. In previous work, Benzm\"uller and
colleagues also studied actualist quantifiers \cite{C37,J31}, and it
should be possible to transfer these ideas to our setting.

Another advantage of the SSE approach, when implemented within
powerful proof assistants such as \isabellehol, is that proof
construction, interactive or automated, can be supported at different
levels of abstraction. For this note that proof protocols/objects may
generally serve two different purposes: (a) they may provide an
independently verifiable explanation in a typically well-defined
logical calculus, or (b) they may provide an intuitive explanation to
the user why the problem in question has been answered positively or
negatively.  Many reasoning tools, if they are offering proof objects
at all, do generate only objects of type (a).  The SSE approach,
however, has already demonstrated its capabilities to provide both
types of responses simultaneously in even most challenging logic
settings.  For example, a quite powerful, abstract level theorem
prover for hyper-intensional higher-order modal logic has been
provided by Kirchner and colleagues \cite{J50,J47}. He encoded, using
abstraction layers, a proof calculus for this very complex logic as
proof tactics and he demonstrated how these abstract level proof
tactics can again be automated using respective tools in
\isabellehol. Kirchner then successfully applied this reasoning
infrastructure to reveal, assess and intuitively communicate a
non-trivial paradox in Zalta's \emph{``Principia Logico-Metaphysica''}~\cite{zalta16:_princ_logic_metap}.

Drawing on the results and experiences from previous work, the
ambition of our ongoing project is to further extend the already
existing portfolio of deontic logics in \isabellehol\ towards a most
powerful, flexible and scalable deontic logic reasoning
infrastructure. A core motivation thereby is to support empirical
studies in various application scenarios, and to assess and compare
the suitability, adequacy and performance of individual deontic logic
solutions for the engineering of ethically and legally regulated agents and explainable intelligent systems.  It is relevant to mention that proof automation
in \isabellehol, and also in related higher-order ATP systems such as
\leoIII~\cite{Leo-III}, is currently improving at good pace.  These
developments are fostered in particular by recently funded research
and infrastructure projects.\footnote{Prominent example projects
  include Matryoshka (\url{http://matryoshka.gforge.inria.fr}) and
  ALEXANDRIA
  (\url{http://www.cl.cam.ac.uk/~lp15/Grants/Alexandria/}).}

\subsection{Discussion}

We propose higher-order logic as a {\em uniform and highly expressive
  formal framework} to represent and experiment with normative
theories of ethico-legal reasoning. To some researchers, this may seem
paradoxical for two reasons. First of all, we do no longer aim for a
unique and standard deontic logic which can be used for all
applications, but we do propose to use higher-order logic as a unique
and formal framework to represent normative theories. So what exactly
is the difference between a unique deontic logic and a unique formal
framework? The second apparent paradox is that we propose higher-order
logic for tool support, whereas it is well-known that higher-order
logic has theoretical drawbacks and is undecidable.

These two apparent paradoxes can be explained away by our {\em
  methodology} of representing normative theories in higher-order
logic. There are many ways in which a normative theory can be
represented in higher-order logic, and only a few of them will be such
that the formal methods of the tool support can be suitably applied to
them. Therefore, the representation of deontic logics in higher-order
logic is an art, and each new representation has to come with a proof
that the embedding is faithful.  These proofs play a similar role in
our formal framework as soundness and completeness proofs play in most
of the traditional work of deontic logic.

\section{Further research} \label{sec:FurtherWork}

In future work, the range of normative theories must be extended, and
the currently represented theories must be further optimized.  In
particular, a wider range of explicit ethical theories must be studied
and formalized.  We have made historical and current developments in
normative reasoning practically accessible for the use in machine
ethics. We showed how our approach can support research in normative
reasoning itself. The use of computer-assisted exploration and
assessment of new deontic logics provides immediate feedback from
systems to property checks. This is particularly valuable for unifying
and combining different logics and for experimental studies.  For
example, since our approach supports meta-logical investigations,
conjectured relationships between I/O logics and conditional logics
can be formally assessed in future work.

Moreover, our approach can also be used for other relevant purposes.
In education, for example, the different logics discussed in this
article can now be integrated in computer-supported classroom
education.  First reassuring results have been obtained in lecture
courses at University of Luxembourg (\emph{Intelligent Agents II}) and
Freie Universit\"at Berlin (\emph{Universal Logical Reasoning} and
\emph{Computational Metaphysics}).  Students start exploring existing
deontic logics and other non-classical logics without the need for a
deep a priori understanding of them. They can use the logics in small
or larger examples, modify them, assess their modifications, etc.

In agent simulation, we plan the embodiment/implementation of explicit
ethical theories in simulated agents. We can then investigate
properties for single agents, but beyond that also study agent
interaction and the behaviour of the agent society as a whole.  To
this end the agent-based \textsc{LeoPARD} framework \cite{LeoPARD},
which is underlying the \leoIII\ prover, will be adapted and utilized
in future work.

\section{Conclusion} \label{sec:Conclusion}

It is the availability of powerful systems such as \isabellehol\ or
\leoIII\ for {\em tool support} that allows our approach to
revolutionize the field of formal ethics.  Though the use of
higher-order logic may come as a paradigm shift to the field of
ethical reasoning, it is an insight which is already well established
in the area of formal deduction.  Whereas it is far from
straightforward to represent deontic logics in higher-order logic,
once a deontic logic has been represented, it becomes much easier to
make small changes to them and see the effect of these changes---and
this is exactly how our approach supports the design of normative
theories of ethico-legal reasoning. It is in the ease in which the
user can work with and adapt existing theories, how the design of
normative theories is made accessible to non-specialist users and
developers.
 
To validate our approach we have embedded the main strands of current
deontic logic within higher-order logic, and we have experimented with
the approach over the past two years.

Our \logikey\ normative reasoning framework and infrastructure
supports empirical studies on ethico-legal theories in which the
underlying logic formalisms itself can be flexibly varied, assessed
and compared in context. This infrastructure can fruitfully support
the development of much needed logic based approaches towards ethical
agency. The solution we have presented supports a wide range of
specific deontic logic variants, and it also scales for their
first-order and higher-order extensions.

The use of tool support for ethico-legal reasoning is not only
fruitful to develop new normative theories, it is now being employed
also in teaching, and we plan to use it as a formal framework for
simulation. Another promising application is the use of our approach
for the study of deontic modality in natural language processing. In
linguistics the use of higher-order logic is already adopted for the
semantics of natural language, and we believe that our framework can
also support studies of the pragmatic aspects of the use of deontic
modality.

\paragraph{Acknowledgments:} We thank the anonymous reviewers for
their valuable feedback and for various comments, which helped to
significantly improve this paper. We also thank Ali Farjami, David
Fuenmayor, Tobias Glei{\ss}\-ner, Alexander Steen, Valeria Zahoransky and several further colleagues for their contributions to this project.

\section*{References}

\end{document}